\documentclass[final,5p,times,twocolumn]{elsarticle}
\biboptions{compress}
\journal{Automatica}

\usepackage{amssymb}

\usepackage{amsmath,amsfonts}
\usepackage{algorithmic}
\usepackage{algorithm}
\usepackage{array}
\usepackage{textcomp}
\usepackage{stfloats}
\usepackage{url}
\usepackage{verbatim}
\usepackage{graphicx}

\usepackage{amsmath,amssymb,amsfonts,amsthm} %,amsthm
\usepackage{xcolor}

\allowdisplaybreaks[1]

\newtheorem{theorem}{Theorem}
\newtheorem{lemma}[theorem]{Lemma}
\newdefinition{ass}{Assumption}
\newdefinition{rem}{Remark}
\newdefinition{defn}{Definition}

\newtheorem{corollary}{Corollary}
\begin{document}	
	\begin{frontmatter}
		\title{No-regret learning for repeated non-cooperative games with lossy 
			bandits\tnoteref{label1}}
		\tnotetext[label1]{The paper was sponsored by Shanghai Sailing Program under 
			Grant Nos. 20YF1453000 and 20YF1452800, the National Natural Science 
			Foundation 
			of 
			China under Grant No. 62003239 and 62003240, and the Fundamental 
			Research Funds 
			for 
			the Central Universities, Shanghai Municipal Science and Technology 
			Major 
			Project 
			No. 2021SHZDZX0100, and Shanghai Municipal Commission of Science and 
			Technology 
			Project No. 19511132101.}
				
		\author[1]{Wenting Liu}\ead{liuwenting@tongji.edu.cn}
		\author[1,2]{Jinlong Lei}\ead{leijinlong@tongji.edu.cn}
		\author[1,2]{Peng Yi\corref{cor1}}\ead{yipeng@tongji.edu.cn}
		\author[1,2]{Yiguang Hong}\ead{yghong@iss.ac.cn}
		
		\cortext[cor1]{Corresponding author}

%		\affiliation[1]{organization={Department of Control Science and 
%		Engineering},%Department and Organization
%			addressline={Tongji University}, 
%			city={Shanghai},
%			postcode={201804}, 
%			country={China}}
%		
%		\affiliation[2]{organization={Department of Control Science and 
%				Engineering},%Department and Organization
%			addressline={Tongji University}, 
%			city={Shanghai},
%			postcode={201804}, 
%			country={China}}
%		
%		\affiliation[3]{organization={Department of Control Science and 
%				Engineering},%Department and Organization
%			addressline={Tongji University}, 
%			city={Shanghai},
%			postcode={201804}, 
%			country={China}}
		
		\address[1]{Department of Control Science and Engineering,
			Tongji University,
			Shanghai 201804, China}            
		\address[2]{Shanghai Research Institute for
			Intelligent Autonomous Systems, Shanghai 201210, China}

		\begin{abstract}
			This paper considers  no-regret learning for repeated continuous-kernel
			games
			with  lossy bandit
			feedback.
			Since it is difficult to give the explicit model of  the utility  
			functions   in
			dynamic  environments, the players' action can only be learned with
			bandit
			feedback.
			Moreover,  because of unreliable communication channels or privacy 
			protection,
			the bandit
			feedback   may be
			lost  or dropped at random.
			Therefore, we study the asynchronous online learning strategy
			of the players
			to adaptively adjust the next actions for minimizing the  long-term 
			regret loss.
			The paper provides a novel no-regret learning algorithm,
			called Online Gradient Descent with lossy bandits (OGD-lb).
			We first give the regret analysis for  concave games with
			differentiable
			and Lipschitz utilities.
			Then we  show that the action profile converges to a Nash equilibrium 
			with
			probability 1
			when the game is also strictly monotone.
			We further  provide the mean square convergence rate
			{$\mathcal{O}\left(k^{-2\min\{\beta, 1/6\}}\right)$} when
			the game is $\beta-$strongly monotone.
			In addition, we extend the algorithm to the case when the loss
			probability of
			the bandit feedback is unknown, and  prove  its almost sure convergence 
			to Nash 
			equilibrium  for strictly monotone games.
			Finally, we take the  resource management in fog computing as an
			application
			example,	and carry out numerical experiments  to empirically  
			demonstrate the 
			algorithm performance.		
		\end{abstract}

	\end{frontmatter}
	
	
	\section{Introduction}

	Online learning is an effective and necessary method for adaptive 
	decision-making in
	dynamical or antagonistic environments. In this case, the agent usually needs to
	select an action without comprehensive models {and then adapt to the 
		next 
		action 
		based on the feedback information it receives.
		Such online
		learning methods are widely used as a central and canonical solution in 
		various 
		fields such as
		online  recommendation \cite{hazan2016introduction}, traffic  routing 
		\cite{mcmahan2013ad}, network resource allocation,  and  
		market prediction 
		 \cite{LESAGELANDRY2020108771,8027140,shalev2012online}.}
		  The online 
	learning
	algorithms generally are designed to minimize the performance metric known as
	\textit{regret}, which is the difference between the cumulative utility incurred 
	by
	online decisions and that of the best-fixed decision in hindsight. {A 
		learning
		algorithm  performs well if it meets the \textit{no-regret} property, i.e.,
		the increase of the regret is sublinear verse
		time  \cite{xu2020distributed,ZHANG2022110006}. In other words, this 
		property 
		means} that the 
	average 
	accumulated regret
	approaches zero asymptotically.  {A wide range of gradient-based 
		no-regret learning 
		algorithms has been established,} e.g., 
	online mirror descent  \cite{NIPS2006_1cfead99,YUAN2018196}, exponential  
	weights  \cite{arora2012multiplicative},
	follow-the-regularized-leader
	  \cite{kalai2005efficient}, 
	online gradient
	descent  \cite{zinkevich2003online,hazan2007logarithmic,SIMONETTO2021109767},  
	etc.

	When a group of agents are interacting in a dynamic environment, each agent's
	utility is not only influenced by its own action but also affected by the 
	actions of
	its opponents. It is desirable that the self-interested agents can produce ideal
	collective behavior patterns by reaching  equilibrium, and even obtain the best
	performance at the system level. \textit{Game theory} provides tools and 
	frameworks
	for the decision-making of \textit{non-cooperative} agents (also called
	players). 	
	Algorithms for game equilibrium seeking have been studied by different methods
	such as the  alternating direction method of
	multipliers \cite{salehisadaghiani2019distributed,8362703},
	the forward-backward operator 
	splitting \cite{YI2019111,9525284,franci2021training},
	discounted mirror descent \cite{9294126}, the iterative
		Tikhonov regularization \cite{9303804},  and
	average consensus protocol \cite{ye2018nash,ZENG201920}, 
	etc.

	For the online decision-making with non-cooperative agents, the learning process 
	can
	be modeled as playing {\it repeated stage games} among a group of players
	 \cite{cesa2006prediction}. In addition, {according to the different 
		types 
		of feedback 
		information, numerous learning algorithms were developed respectively.}
	In some situations, the
		gradient information or second-order Hessian of the player's own utility 
		function 
		can be obtained after all
	nodes select actions, with which the next round action can be updated for
	maximizing its own utility.
	Since gradient feedback relies on full manipulation of
	its own utility function, it is called \textit{full-information}. 
	For this scenario, various
		convergence and regret analyses  have been derived. For example, A 
		no-regret 
	algorithm is proposed in
	 \cite{gordon2008no}, which guarantees the convergence to the correlated 
	equilibria
	in the repeated convex game.
	 \cite{daskalakis2011near} {focused on zero-sum games and showed that 
		the 
		actions of the players generated by a no-regret algorithm called NoRegretEgt 
		converge to a min-max equilibrium.}
	An adaptive regret-based learning procedure has been
	applied to track the correlated equilibria set of the congestion game
	 \cite{maskery2009decentralized}. When the gradient feedback  is lost randomly,
	 \cite{NEURIPS2018_10c66082} focused on  variationally
		stable games and showed that the online gradient descent algorithm  converge 
		almost surely to 
		the set of Nash equilibria.

	{\it Nevertheless, in many practical problems, utility
		functions
		describing performances
		like service latency or reliability are difficult to model with explicit 
		form in
		dynamic environments. Moreover, some low-power devices cannot run complicated
		models
		such as deep neural network to derive gradients.} Besides, the player may 
	not even be
	aware of the existence of its opponents. In these settings, the only information
	that
	the player can obtain after choosing an action is its utility value, which is 
	known
	as \textit{bandit feedback}. How to making online decision with such limited
	information is our concern.

	With bandit feedback, players need to derive an individual gradient
	estimate from
	the utility value to update the next action. The most commonly used
	methods for gradient estimation can be divided into two
	types: multi-point estimation and single-point estimation. In fact,
	\textit{multi-point
		estimation} techniques have been widely used in various optimization problems
	 \cite{nesterov2017random,8556020,9222230,9216151,cao2021decentralized,zhu2021hessian},
	which are also
	known as zeroth-order oracles. Different from optimization, for the repeated game
	problem, each player's utility is not only related to the actions taken by itself
	but also related to the actions taken by its opponents. When its actions change, 
	the
	actions of its opponents will also change accordingly. Therefore, multi-point
	estimation methods are not applicable or cost too much, hence, 
	\textit{single-point
		estimation }methods such as simultaneous perturbation stochastic 
		approximation
	(SPSA)
	 \cite{spall1997one,5439936,10.5555/1070432.1070486,NEURIPS2020_b3d6e130} are 
	studied
	in the context of non-cooperative games. For example, the work of 
		 \cite{NIPS2017_39ae2ed1} considered  potential games and proved that 
		the exponential weight
		learning program with bandit feedback can achieve a sublinear expected 
		regret 
	and
	converge to  Nash equilibrium.
	 \cite{NEURIPS2018_47fd3c87} showed that in
	monotone concave games, no-regret learning based on mirror descent with bandit
	feedback can converge to a Nash equilibrium with probability 1. With delayed 
	bandit
	feedback in monotone games,  \cite{heliou2020gradient} proposed an algorithm with
	sublinear regret and convergence to a Nash equilibrium.  \cite{shi2019no} showed 
	that the no-regret learning in the Cournot game with bandit feedback can 
	converge to the unique Nash equilibrium.

	However, {\it 
		random loss and drop of the bandit 	feedback can  occur in practical 
			scenarios.} Because the utility value  evaluated by the external dynamic 
		environment  might be
	lost during transmission.
	Moreover, many computing architectures rely on the
	support of communication networks. Once the communication channel is interrupted,
	services and information feedback are dropped. Besides, device mobility can cause
	random variations in channel quality, aggravating communication channel failures
	 \cite{9578933}. In addition, to protect privacy 
	or to perform intermittent queries to  reduce the  query costs, the 
	bandit 
	feedback 
	can be 
	actively dropped at random.
	Compounding bandit feedback with lossy feedback deserves in-depth studies, which 
	is
	still lacking in the literature  \cite{xu2020distributed}.

	Consider \textit{fog computing} as an application example, which is a
	distributed
	computing
	architecture for the Internet of Things (IoT). First of all, in the management of
	resources (including CPU time and storage) in fog computation, online learning 
	for
	adaptive decision-making is an effective and necessary method. On the one hand,
	since the fog computing architecture targets at latency-sensitive IoT 
	applications,
	real-time decision-making through online learning is an effective approach to
	improve user experience. On the other hand, {\it the dynamic or noncooperative 
		opponent
		IoT users are so complicated that we cannot build a comprehensive model}. In 
		this
	case, we adapt the decision through {\it online learning with bandit
		feedback}. Because the utility functions describing device
	latency and reliability in IoT are often difficult to establish explicitly, and
	low-power edge devices in fog computing cannot provide the computing power 
	required
	for gradient computing. Adaptive online decision-making methods have also been
	highlighted in various problems of fog computing, eg, online computation 
	scheduling
	 \cite{xu2020gradient}, computation
	offloading  \cite{shen2019computation}, and resource allocation
	 \cite{hazan2007logarithmic,8882321}.
	An online bandit saddle-point (BanSaP) scheme for IoT management is developed in
	 \cite{chen2018bandit}, which can achieve sublinear dynamic regret and can deal 
	with
	time-varying
	constraints based on multi-point bandits.

	Motivated by the above, we focus on no-regret learning with lossy 
		bandits for 
		repeated continuous-kernel games and take the resource management
		in fog computation as an application example. A preliminary version of the 
		results was presented at the IEEE CDC in 2021
	 \cite{liu2021no}. The current work makes {several} improvements and 
	extensions compared to  \cite{liu2021no}: 
	the major one we would {like to emphasize} 
	is that we derive a
	convergence rate that can reach the same order of bandit feedback in
	 \cite{NEURIPS2018_47fd3c87} without information loss. In addition, we further 
	relax
	the assumption to consider the case where the probability of bandit feedback 
	loss is
	unknown. In this more
	practical and complex case, we demonstrate the convergence of the algorithm, for
	which the analysis is not intuitive. Furthermore, we take fog computation as
	an application example and carry out more simulations   to discuss how loss
	probability influences the
	number of iterations and the times of updates required for the algorithm to 
	reach a
	certain
	accuracy. The main contributions of our work are summarized  as follows:
	\begin{enumerate}[1)]
		\item We propose a  novel
		no-regret learning algorithm capable of online decision-making with lossy
		single-point bandit,
		called \textit{ Online Gradient Descent with lossy bandits
			(OGD-lb)}.
		\item 
		We derive the expected regret bound of the learning algorithm, and show 
		that it conforms to the no-regret property with concave utilities for proper 
		step-sizes.
		\item
		We show that OGD-lb converges to a Nash equilibrium
		with probability 1 for strictly monotone games, and it
		achieves 	
		{$\mathcal{O}\left(k^{-2\min\{\beta, 1/6\}}\right)$} convergence rate
		for $\beta-$strongly monotone games.
		It is worth noting that this convergence rate reaches the same order of
		bandit feedback in  \cite{NEURIPS2018_47fd3c87} without information loss.
		\item We also consider the case when the
		probability of
		bandit feedback loss is unknown. The step-size is set by counting the number
		of players updates up to the current moment. We show
		that
		the algorithm still converges to a Nash equilibrium
		with probability 1 for strictly monotone games.
	\end{enumerate}

	The paper is organized as follows.  We state the
	problem formulation in   Section \ref{formulation} and introduce
	the  algorithm in Section \ref{sec-alg}. The main results  and proofs are 
	provided 
	in Sections 
	\ref{main-res} and \ref{proof}, respectively. Section \ref{simu} presents 
	simulation 
	results. Some concluding remarks are provided in Section
	\ref{conclu}.

	\textit{Notations:} Denote $i\in\mathcal{N}=\{1,2,\dots,N\}$ as the player in 
	the 
	game, and  $k = 1,\dots,K$ as iterations. The indicator function of player $i$ 
	at 
	iteration
	$k$ is denoted by $I_{i}^{k}$. The 
	$m$-dimensional real Euclidean space is denoted by $\mathbb{R}^{m}$.  
	%The Cartesian product of the sets $\{\mathcal{A}_{i}\}_{i=1,\dots,N}$ is 
	%represented 
	%by $\prod_{i=1}^{N} \mathcal{A}_{i}$.
	%Constants are denoted with upper case
	%letters, i.e., $A,B,C$, etc.
	Sets are denoted by calligraphy, i.e.,
	$\mathcal{A},\mathcal{B},\mathcal{C}$, etc. 
	For a column vector $x\in\mathbb{R}^{m}$,  $x^{\top}$ denotes its transpose
	$x^\top y = \left\langle x ,
	y\right\rangle $ denotes the inner product of $x, y$, {and the standard 
		Euclidean 
		norm is denoted by $\|x\|=\sqrt{x^\top x}$. 
		Denote $\|x\|_{\infty}=\max_{1\leq i\leq N}|a_{i}|$  as the 
		max
		norm.} Use $P_{\mathcal{A}}(y)=\arg
	\min_{a\in\mathcal{A}}\|y-x\|$ to represent 
	the projection of $y$
	onto a closed convex set $\mathcal{A}$.   For  functions
	$f, \phi:\mathbb{R}\rightarrow\mathbb{R}^+$, we write 
	$f(x)=\mathcal{O}(\phi(x))$ if 
	$\lim 
	\sup_{x\rightarrow\infty}\left| f(x)/\phi(x)\right| <\infty$,
	and  $f(x)=o(\phi(x))$ if $\lim \sup_{k\rightarrow\infty}f(x)/\phi(x)=0$.

\section{Problem formulation}\label{formulation}

{In this section, a repeated concave
	game is   formulated.    Moreover,    the definition of regret is introduced, 
	which is a   performance metric for the online learning algorithm.}

\subsection{ Repeated Concave Games}\label{concave-game}

The tuple $\mathcal{G} \equiv \mathcal{G}(\mathcal{N},
\mathcal{A} \equiv \prod_{i=1}^{N} \mathcal{A}_{i},
\left\{u_{i}\right\}_{i=1}^{N})$ denotes a utility maximization game, where 
$\mathcal{N}=\left\lbrace1,\dots,N
\right\rbrace $ is the set of $N$ agents/players. 
$\mathcal{A}_{i}\subset \mathbb{R}^{d_{i}}$ is the action space of player $i$. And 
$u_{i}(a)=u_{i}(a_{i},a_{-i}):\mathcal{A} \rightarrow \mathbb{R}$ is player $i$'s
utility function. Let
$a_{-i}=(a_1,\cdots, a_{i-1},a_{i+1},\cdots a_{N})$ and $\mathcal{A}_{-i}=\prod_{j
	\neq i} \mathcal{A}_{j}$  represent
the actions and action space for all
players except $i$, respectively.
The  action profile is denoted as $a=(a_{i},a_{-i})$.
{Our basic assumptions about the utility functions and the action sets 
	are 
	as follows.}

\begin{ass}\label{ass-lip}
	For each player $i\in\mathcal{N}$,
	\begin{enumerate}[i)]
		\item the action set $\mathcal{A}_i$ is 
		closed, convex, and  compact
		with a
		nonempty interior.
		\item $u_{i}(a_{i},a_{-i}) $ is concave and continuously differentiable in 
		$a_i \in \mathcal{A}_{i}$ for any given
		$a_{-i}    \in \mathcal{A}_{-i}$; 
		\item  {$
			g_{i}(a_{i},a_{-i})=\nabla_{a_{i}} u_{i}\left(a_{i} , a_{-i}\right)$ 
			represents 
			the partial gradient of $u_{i}(a_{i},a_{-i})$ with respect to $a_i$,}
		which is 	$L_i$-Lipschitz continuous {in $a\in\mathcal{A} $}, i.e.,
		\begin{align*}
		\left\|g_{i}\left(a\right)-g_{i}(a^{\prime})\right\|
		\leq	L_i\left\|a-a^{\prime}\right\|,\quad \forall 
		a,a^{\prime}\in\mathcal{A}.
		\end{align*}		
	\end{enumerate}		
\end{ass}

Time is slotted as  $k=1,2,\dots$, {we assume that   the players repeatedly play
	the
	game  $  \mathcal{G}(\mathcal{N},
	\mathcal{A} \equiv \prod_{i=1}^{N} \mathcal{A}_{i},
	\left\{u_{i}\right\}_{i=1}^{N})$}.
Each player $i$   {sequentially chooses  its action by learning from the
	available
	feedback information.}
The algorithm for adapting the player's action is called {\it ``online learning"}.

\subsection{Regret of Online Learning Algorithm}

Regret is usually taken as the metric  to measure the performance of online
learning
algorithms. The learning protocol of a given repeated game is as follows:
	At iteration $k$, each player $i$ selects
an action $a_{i,k}$ through a learning algorithm.
Then the external environment such as  the app user market evaluates the current 
action
profile $(a_{i,k},a_{-i,k})$, and returns  the value
$u_{i}\left(a_{i,k},a_{-i,k}\right)$ to player $i$.
The cumulative utility of player $i$ within $K$ iterations is denoted as
$\sum_{k=1}^{K}u_{i}\left(a_{i,k},a_{-i,k}\right)$.
To measure the performance of $\mathcal{A}$, the cumulative utility  is
usually compared with the utility obtained when a best-fixed decision in the
hindsight is taken.  {The regret of node $i$  within
	$K$ iterations is formally defined as}
\[
\mathcal{R}eg^{(i)}\left(K \right)=\max _{a_{i}^{\prime} \in
	\mathcal{A}_{i}} \left\lbrace  \sum_{k=1}^{K} u_{i}\left(a_{i}^{\prime},
a_{-i,k}\right)-
\sum_{k=1}^{K} u_{i}\left(a_{i,k}, a_{-i,k}\right) \right\rbrace.
\]

An online algorithm is  no-regret if and only if the regret is
sublinear as a function of time $K$, $\mathcal{R}eg^{(i)}(K)=o(K)$, i.e.,
\begin{equation}
\lim_{K \rightarrow \infty} \frac{\mathcal{R}eg^{(i)}(K)}{K}=0,\quad \forall i\in
\mathcal{N}.
\end{equation}

\section{Online learning with lossy bandits}\label{sec-alg}

In this section,  {an online learning algorithm was designed, which is
	called} \textit{Online Gradient Descent with lossy bandits} (OGD-lb).

\subsection{Lossy Bandits}
With bandit feedback, the only information available to the  players
is the utility value with a given action. However,
{the utility value may not be available at each stage of  online
	learning. Take the fog computation networks as an example,}
$a)$ the utility may be lost during transmission; $b)$ the interruption of network
connection; $c)$ communication channel
failures due to small-scale fading or IoT device mobility; $d)$ {the player
	performs intermittent queries to  reduce the  query costs; $e)$ the player
	actively drops the utility
	for privacy protection, etc.}

We consider the lossy bandits scenario as  shown in Figure \ref{lossy-fig}.
At iteration $k$, each player $i\in\mathcal{N}$  submits
its \textit{applied 	action} $\hat{a}_{i,k}$ (obtained after perturbing
\textit{intended action}
$a_{i,k}$) to the  market.
However,  some players may not receive the utility value due to  feedback loss, such
as players 1 and 3 in Figure \ref{lossy-fig}. {Then the   players that have
	received the   utilities
	update their  actions,} while those that cannot receive the utility information
keep their actions unchanged. {For a detailed description of this 
	process, please refer to the algorithm introduction in the next subsection.}

\begin{figure}[htbp]
	\centering
	\includegraphics[width=3in]{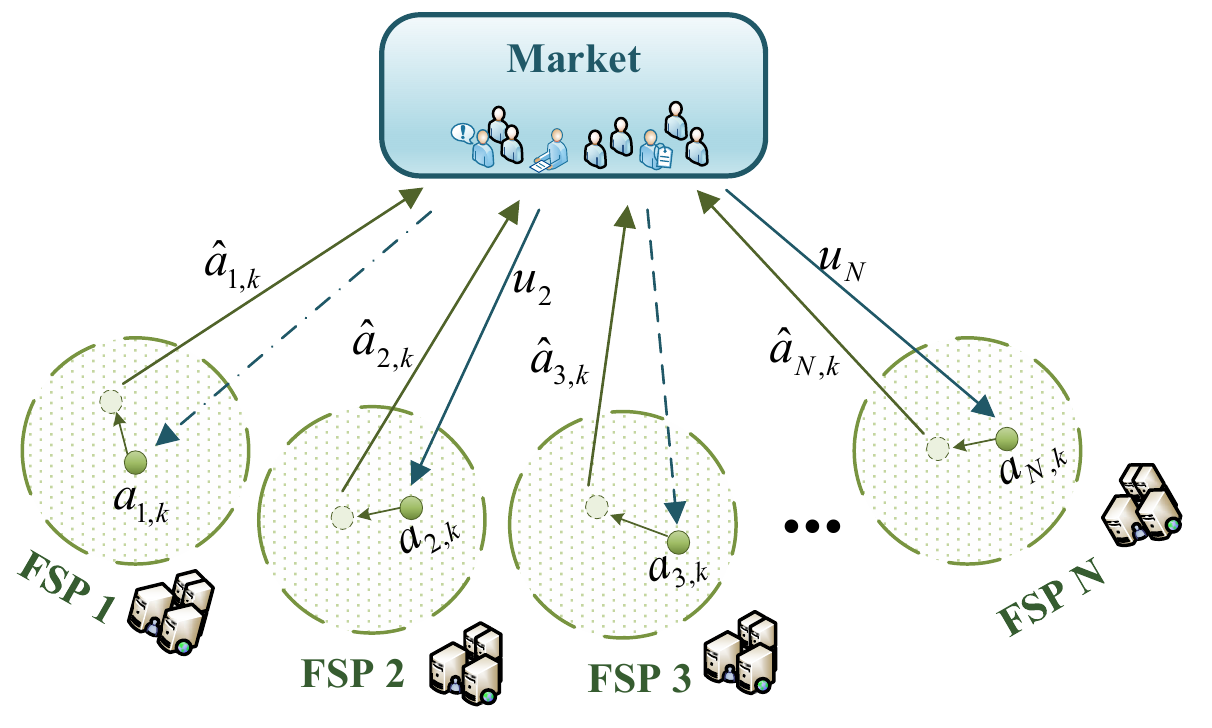}
	\caption{\small Lossy bandits (where FSP stands for fog service provider,  the
		solid line represents the received
		information transmission, and the dashed line represents information loss;
		$a_{i,k}$ is the action of FSP $i$ at iteration $k$; $u_{i}$
		is
		the utility value of FSP $i$).}\label{lossy-fig}
\end{figure}

{Let $p_{i}$ denote the probability that  player $i$  will receive its 
	utility value,  then  the probability of  information loss is $1-p_{i}$.} Denote 
the indicator function as follows:
\begin{align*}
\begin{split}
I_{i}^{k}= \left\{
\begin{array}{lr}
1, \; {\rm if \; the  \; bandit \;  is \;  received;}\\
0, \; {\rm if  \; the \;  bandit \;  is  \; lost.}
\end{array}
\right.
\end{split}		
\end{align*}
Then $\mathbb{E}\left[I_{i}^{k} \right]=p_{i}>0$.

\subsection{Learning Process}

The learning process of our proposed algorithm can be divided into three parts.
The first part is initializing  parameters. The second part is estimating the
gradient
with bandit feedback, in which we use a one-point estimation method
motivating by  simultaneous perturbation stochastic
approximation (SPSA) \cite{spall1997one}. The third part is performing projected 
gradient descent.

At each stage, the players update their actions by the novel algorithm 
	called 
	\textit{Online 
		Gradient Descent with
		lossy bandits} (OGD-lb)  (Algorithm \ref{alg1}), where  the action of 
		player 
	$i$ at stage $k$ is denoted by $a_{i,k}$. A detailed description of the 
	algorithm is shown below and the corresponding position with the 
	\textit{pseudo-code} is
	marked in parentheses.

\begin{itemize}
	\item \textbf{(Initialization)} Set $k=1$, require step-size 
	$\{\gamma_{i,k}\}>0$
	and query radius  $\{\delta_{k}\}>0$ are non-increasing sequences, 
	choose an action
	$a_{i,k}\in \mathcal{A}_{i}$ for each player $i\in\mathcal{N}$.
	\item\textbf{(Line 4-5)}  {Fix a
		$\delta_{k}>0$ and select a
		vector   $\lambda_{i,k}$ from the unit sphere
		$\mathbb{S}_{i}\equiv\mathbb{S}^{d_{i}} \subset \mathbb{R}^{d_i}$
		that is independent with each other at stage $k$.}  To ensure that the
	perturbed point is still in
	the action space $\mathcal{A}_{i}$, we select an interior point $c_{i} $ from $
	\mathcal{A}_{i}$ and let $\mathbb{B}_{r_{i}}\left( c_{i}\right) $ be a
	$r_{i}$-ball centered at $c_{i}\in \mathcal{A}_{i}$ so that
	$\mathbb{B}_{r_{i}}\left( c_{i}\right)\subseteq \mathcal{A}_{i}$. We then take
	\begin{align}\label{safe-net}
	\theta_{i,k}=\lambda_{i,k}-r_{i}^{-1}\left(a_{i,k}-c_{i} \right)
	\end{align}
	as the \textit{perturbation direction}.
	\item\textbf{(Line 6)} Get an \textit{applied action}
	$\hat{a}_{i,k}=a_{i,k}+\delta_{k}\theta_{i,k}=a_{i,k}^{\delta_{k}}+\delta_{k}
	\lambda_{i,k}$ to
	play, where
	\begin{align}\label{new-pivot}
	a_{i,k}^{\delta_{k}}=a_{i,k}-r_{i}^{-1}\delta_{k}\left(a_{i,k}-c_{i}
	\right)
	\end{align}
	with $\delta_{k}/r_{i}<1$.
	Note that it is equivalent to first
	moving each intended action $a_{i,k}$ to $a_{i,k}^{\delta_{k}}$, and then
	perturbing
	along the direction $\lambda_{i,k}$ to get the applied action $\hat{a}_{i,k}$.
	
	\item\textbf{(Line 7-8)} After that, we obtain the utility value
	$\hat{u}_{i,k}=u_{i}\left(\hat{a}_{i,k},\hat{a}_{-i,k}
	\right)$, and {derive an estimated gradient by}
	\begin{align}\label{es-gradient}
	\hat{g}_{i,k}\left({a}_{i,k},{a}_{-i,k}
	\right)=\frac{d_{i}}{\delta_{k}}\hat{u}_{i,k}\lambda_{i,k}.
	\end{align}
	\item\textbf{(Line 9)} Finally, if $I_{i}^{k}=1$,
	player $i$  updates its action $a_{i,k+1}$ by the projected gradient
	method.
	Otherwise, $a_{i,k}$ remains unchanged.
\end{itemize}

In summary, we provide a novel algorithm    OGD-lb,  for which the
pseudo-code is shown in Algorithm 1. 	For convenience, we  abbreviate
$\hat{g}_{i}\left({x}_{i,k},{x}_{-i,k}
\right)$ as $\hat{g}_{i,k}$ in the rest.

\begin{algorithm}\caption{  OGD-lb}\label{alg1}
	\textbf{Require:} step-size
	$\{\gamma_{i,k}\}>0$, query radius
	$\{\delta_{k}\}>0$,
	safety ball $\mathbb{B}_{r_{i}}\left(c_{i}\right) \subseteq
	\mathcal{A}_{i}$   \\
	1: choose $a_{i,k} \in \mathcal{A}_{i}$, iteration $k
	\leftarrow 1$  \\
	2: \textbf{repeat}   \\
	3: \quad \textbf{for} each player $i$ \textbf{do} \\
	4: \quad \quad draw $\lambda_{i,k}$ uniformly from $S_{i} \equiv
	S^{d_{i}}$ of $\mathbb{R}^{d_{i}}$ \\
	5: \quad \quad set $\theta_{i,k} \leftarrow
	\lambda_{i,k}-r_{i}^{-1}(a_{i,k}-c_{i})$  \\
	6: \quad \quad play $\hat{a}_{i,k} \leftarrow a_{i,k}
	+\delta_{k}
	\theta_{i,k}$    \\
	7: \quad \quad receive  $\hat{u}_{i,k}=	u_{i}\left(\hat{a}_{i,k} ,
	\hat{a}_{-i,k}\right)$ \\
	8: \quad\quad set $\hat{g}_{i,k} \leftarrow\left(d_{i} /
	\delta_{k}\right) \hat{u}_{i,k}  \lambda_{i,k}$  \\
	9: \quad \quad update \\
	$$\quad \quad\quad  a_{i,k+1} \leftarrow
	\left\{\begin{array}{ll}
	P_{\mathcal{A}_{i}}\left(a_{i,k} +
	\gamma_{i,k} \hat{g}_{i,k} \right),
	& \text {
		if }
	I_{i}^{k}=1 \\
	a_{i,k}, & \text { if } I_{i}^{k}=0
	\end{array}\right.$$
	10: \quad \textbf{end for}  \\
	11: \quad $k \leftarrow k+1$  \\
	12: \textbf{until} end  	
\end{algorithm}

{ With respect to the independence
	of the perturbation
	sequences $\{\lambda_{i,k}\}$ and  the indicator functions $\{I_{i}^{k}\}$ used 
	in 
	Algorithm 1, the following assumption is made.
}

\begin{ass}\label{ass-zi-ii}
	At each stage $k$, the random variables
	$  \lambda_{i,k} ,  I_{i}^{k}  ,i=1,\cdots, N $ are mutually independent. 
	In addition, for  each 
	$i\in\mathcal{N}$,   $ \{I_{i}^{k}\}_{ k\geq 1}$
	are   independent and
	identically distributed (\textbf{i.i.d.})    across time steps.
\end{ass}

\section{Main Results}\label{main-results}\label{main-res}
In this section,{ the main theoretical results are given, which contain 
}the expected 
regret bound, convergence, and convergence rate with different
algorithm step-size selections.

\subsection{Regret Analysis}

Here, we demonstrate the expected regret bound of Algorithm 
\ref{alg1}  and  show 
that it meets no-regret property. For detailed proofs of Theorem \ref{regret} and 
Corollary
	\ref{no-regret}, please refer to  Section\ref{the1} and
		\ref{cor1}.

\begin{theorem}[Regret bound in expectation]\label{regret}
	Let Assumptions  $\ref{ass-lip}$-$\ref{ass-zi-ii}$ hold.
		Consider the players follow Algorithm 1 with step-size 
	$\gamma_{i,k}=\gamma_{k}p_{i}^{-w}$ 
	and constant
	$w>0$. Suppose that
	$\{\gamma_{k}\}$
	and   $\{\delta_{k}\}$ are  non-increasing sequences
	with
	$\delta_{1}<\min_{i} r_{i}$.
	Then for each 	$i\in \mathcal{N}$, 
	\begin{align}
	\mathbb{E}\left[\mathcal{R}eg^{(i)} (K)\right]
	&\leq
	\frac{p_{i}^{w-1}B_{i}^{2}}{2\gamma_{K}}
	+\sum_{k=1}^{K}B_{i}L_{i}\sqrt{N}\delta_{k}\notag\\
	&+\sum_{k=1}^{K}\frac{p_{i}^{w-1}}{2\gamma_{k}}\frac{G_{i}^{2}}{\delta_{k}^{2}}
	\mathbb{E}\left[\gamma_{i,k}^{2} \right],
	\end{align}
	where 
	$B_{i}=\max_{a_{i}^{\prime},a_{i}\in\mathcal{A}_{i}}\|a_{i}^{\prime}-a_{i}\|$
	{represents} the Euclidean diameter of $\mathcal{A}_{i}$, and
	$G_{i}^{2}=d_{i}^{2} \max_{a
		\in\mathcal{A}}\left|u_{i}(a_{i},a_{-i})\right|^{2}$ is a
	bounded	constant.
\end{theorem}

\textit{Theorem \ref{regret}} 
proves that  the number 
of  players $N$, {and the algorithm update parameters} (step-size 
$\gamma_{i,k}$ and perturbation
radius $\delta_{k}$) of player $i\in\mathcal{N}$ at iteration $k$ 
all affect the {expectation-valued} 
regret 
bound of the  Algorithm 1.
Then, {for  some specific} step sizes,  the 
no-regret  
property is proved
in the following corollary.

\begin{corollary}[Expected No-regret]\label{no-regret}
	{Suppose} Assumptions \ref{ass-lip}-\ref{ass-zi-ii} hold. 
		Consider the players follow Algorithm 1 with 
	$\gamma_{i,k}=1/(k^{b}p_{i})$ and
	$\delta_{k}={\delta_{1}}k^{-c} $
	with  {$0<b<1$,  $0<c<b/2$} and {$\delta_{1}<\min_{i} r_{i}$}. Then
	{for each 	$i\in \mathcal{N}$},
	\begin{align}
	\mathbb{E}&\left[\mathcal{R}eg^{(i)} (K)\right]
	\leq \frac{B_{i}^{2}}{2}k^{b}
	+ \frac{B_{i}L_{i}\sqrt{N}{\delta_{1}}}{1-c}K^{1-c}
	\notag\\
	&\quad +\frac{G_{i}^{2}}{2p_{i}^{2}{\delta_{1}^{2}}(1-b+2c)}K^{1-b+2c}
	+B_{i}L_{i}\sqrt{N}{\delta_{1}}
	+\frac{G_{i}^{2}}{2p_{i}^{2}{\delta_{1}^{2}}}. \notag
	\end{align}	
\end{corollary}

\vspace{3mm}
\begin{rem}
	{According to} Corollary \ref{no-regret},  we obtain  that
	$$\mathbb{E}\left[\mathcal{R}eg^{(i)}(K)\right]
	=\mathcal{O}(K^{\max\{b,1-c,1-b+2c\}}).$$
	Thus, $\lim \sup _{K \rightarrow \infty}
	\mathbb{E}\left[\mathcal{R}eg^{(i)}(K)\right]/ K =0,$
	which implies that Algorithm \ref{alg1}
	is  no-regret. Note that it is desirable for the players 
		to follow a 
		no-regret learning algorithm because everyone wishes that the  online 
		strategy he adopted is at least not worse than any static strategy.
	For  example, a regret bound  $K^{3/4}$ can be obtained with 
	$b=3/4$ and
	$c=1/4$, which is a common  bound in the  online learning 
	literature, such
	as \cite{10.5555/1070432.1070486}.
\end{rem}

\subsection{Convergence Analysis}

\begin{defn}\label{defn2}(Nash equilibrium).
	{The profile 
		$a^{*}\in 
		\mathcal{A}$ is a Nash equilibrium for a given  game $\mathcal{G}$} if
	for each $i \in \mathcal{N}$,
	\[u_{i}\left(a_{i}^{*},a_{-i}^{*}\right)\geq u_{i}\left(a_{i},a_{-i}^{*}
	\right),\quad \forall a_{i}\in\mathcal{A}_{i}.\]		
\end{defn}

It is worth noting that a no-regret algorithm cannot
ensure the convergence to the Nash equilibrium in general, for instance,  the
sequence of
actions can converge to the coarsest equilibrium or correlated equilibrium
\cite{cesa2006prediction}. Convergence to a Nash equilibrium is ``considerably
more difficult" because Nash equilibrium is a more stable equilibrium.
To study the convergence of the algorithm, we further restrict the 
	game 
	structure to a  strictly monotone game \cite{rosen1965existence}.
  In the  following, the pseudo-gradient mapping is denoted by 
	$g(a)=(g_{1}(a_{1},a_{-1}),\dots,g_{N}(a_{N},a_{-N}))^\top$.

\begin{ass}\label{defn-monotone}
	{Suppose that $\mathcal{G}$ is} a strictly monotone game on action 
	space
	$\mathcal{A}$, i.e.,
	\[\left\langle g(a)-g(a^{\prime}), a-a^{\prime}\right\rangle  < 0, \quad \forall 
	a,a^{\prime}\in
	\mathcal{A}, a\neq
	a^{\prime}.\]
	
\end{ass}

\begin{rem}\label{rem-vi}
	{ When  the action
		set $\mathcal{A}_{i}$ is convex and compact for each player 
		$i\in\mathcal{N}$, a 
		strictly monotone game}
	admits a unique Nash equilibrium $a^{*}$, which is equivalent to the 
	solution	of 
	the variational inequality\cite{scutari2010convex} 
	\begin{align}\label{vi}
	\sum_{i\in\mathcal{N}}\left\langle g_{i}\left(a_{i}^{*},a_{-i}^{*} \right),
	a_{i}-a_{i}^{*}
	\right\rangle \leq 0, \quad \forall a_{i}\in\mathcal{A}_{i}.
	\end{align}

\end{rem}

The assumption regarding the step-size is  as follows, which can also 
	be found  in the
existing literature, see e.g., \cite{NEURIPS2018_47fd3c87}.

\begin{ass}\label{ass-stepsize}
	For each $i\in\mathcal{N}$, the sequences
	$\{\gamma_{k}\}$ and
	perturbation radius
	$\{\delta_{k}\}$ satisfy { $\delta_{1}<\min_{i} r_{i}$, and}
	
	\[\lim _{k \rightarrow \infty} \gamma_{k}=\lim _{k \rightarrow
		\infty}
	\delta_{k}=0, \quad \sum_{k=1}^{\infty}
	\gamma_{k}=\infty,\]
	\[\sum_{k=1}^{\infty}
	\gamma_{k}
	\delta_{k}<\infty, \quad \sum_{k=1}^{\infty}
	\frac{\gamma_{k}^{2}}{\delta_{k}^{2}}<\infty. \]
\end{ass}

Then, we can obtain the convergence of the algorithm.

\begin{theorem}[Almost sure convergence]\label{the-con}
	\label{convergence}
	Consider
	{Algorithm \ref{alg1} with $\gamma_{i,k}=\gamma_k p_i^{-1}$ }{for all
		$i\in\mathcal{N}$.}
	Let Assumptions \ref{ass-lip}-\ref{ass-stepsize} hold. Then the
	action sequence   $\hat{a}_{k}$ converges to a Nash
	equilibrium with probability 1. 	
\end{theorem}

The results are proved as follows. {Firstly,   with Assumptions
	\ref{ass-lip}-\ref{defn-monotone},
	we prove 	that $D\left(a_{k},a^{*}\right) =\|a_k-a^*\|^2/2$ converges almost
	surely (a.s.) to
	a finite random variable $D_{\infty}$. We then} prove that there  exists a
subsequence
$\{a_{k_{r}}\}$ of $\{a_k\}$ which converges
a.s.  to the Nash equilibrium. Finally, combining the above two results,
we prove  Theorem
\ref{the-con}. Please refer to {Section \ref{theo2}}
for detailed proofs.

\subsection{Rate Analysis}

In order to study the convergence rate of the  proposed algorithm, we further
strengthen the  game structure  into a
$\beta$-strongly monotone game with specific step-sizes.
{For the detailed proof of  Theorem \ref{the2}, please refer to Section
	\ref{theo3}.}

\begin{ass}\label{strongly}
	Suppose that  $\mathcal{G}$ is a $\beta-$strongly monotone game on
	action
	space
	$\mathcal{A}$, i.e.,
	\[\left\langle g\left(a\right)-g(a^{\prime}),
	a-a^{\prime}\right\rangle\leq -\beta \|a-a^{\prime}\|^{2}, \quad \forall 
	a,a^{\prime} \in \mathcal{A}.\]
\end{ass}

\begin{theorem}[Convergence rate in  a mean-squared sense]\label{the2}
	Suppose Assumptions \ref{ass-lip}, \ref{ass-zi-ii}, and \ref{strongly}  hold.
		Consider the players follow Algorithm 1 with  
	$\gamma_{i,k}={1}/({k p_{i}})$  and  	
	$\delta_{k}={\delta_{1}}k^{-1/3} $ with $\delta_{1}<\min_{i}
	r_{i}$.
	Then 
	$$\mathbb{E}\left[\|\hat{a}_{k}-a^{*}\|^{2} \right]
	=\mathcal{O}\left(k^{-2\min\{\beta, 1/6\}}\right).$$	
\end{theorem}

\subsection{Convergence with Unknown  Lossy Probability}

{In addition, we consider the scenario where $p_{i}$ is unknown.}
Let the step-size be
a function related to the number of updates up to the current time $k$, i.e.,
step-size
$\gamma_{i,k}={1}/{(\Gamma_{i}^{k})^{q}}$ where
$\Gamma_{i}^{k}=\sum_{t=1}^{k}I_{i}^{t}$, $q\in (1/2, 1]$, and
$\mathbb{E}[I_{i}^{k}]=p_{i}>0$ for all $i$ and $k$.
{In this setting,  the  symbol}
$p_{i}$ is only used  for analysis.

{Then, we can obtain the convergence of the algorithm when
	the  loss probability is unknown in advance, for which the proof can be found in
	Section \ref{theo4}}.

\begin{theorem}[Almost sure convergence with unknown $p_i$]\label{the-con-pi}
	Suppose Assumptions \ref{ass-lip}, \ref{ass-zi-ii}, and \ref{defn-monotone}
	hold.
		Consider the players follow Algorithm 1, where
	$\gamma_{i,k}={1}/{(\Gamma_{i}^{k})^{q}}$ with
	$\Gamma_{i}^{k}=\sum_{t=1}^{k}I_{i}^{t}$ and $q\in(1/2, 1]$, and  the
	perturbation radius
	$\{\delta_{k}\}$ satisfies { $\delta_{1}<\min_{i} r_{i}$, and}
	\begin{align} \label{ass-delta}\lim _{k \rightarrow \infty} \delta_{k}=0,~
	\sum_{k=1}^{\infty}
	k^{-q} \delta_{k}<\infty, ~  \sum_{k=1}^{\infty} k^{-2q}
	\delta_k^{-2}<\infty  .\end{align}
	Then	the sequence of realized
	actions $\hat{a}_{k}$ converges to the Nash
	equilibrium with probability 1. 	
\end{theorem}

\section{Proof of Main Results}\label{proof}

In this part, we provide detailed proofs  {corresponding to} the main 
results 
established in 
Section~\ref{main-results}.

\subsection{Preliminary Analysis}

Let $\mathcal{F}_{k}$ be  a
$\sigma-$algebra  of random variables up
to stage $k$, i.e., {$\mathcal{F}_{k}\triangleq\sigma\left\lbrace a_{i,1},
	\lambda_{i,s},	I_{i}^{s}, i\in\mathcal{N},  1\leq s\leq k-1\right\rbrace$}. We
denote
\begin{align}\label{es}
\hat{g}_{i,k}=g_{i}\left(a_{k}\right)+\rho_{i,k+1}+\zeta_{i,k}
\end{align}
where \begin{align}
&\rho_{i,k+1}=\hat{g}_{i,k}-\mathbb{E}\left[\hat{g}_{i,k} \mid
\mathcal{F}_{k}\right],\label{ui}
\\&\zeta_{i,k}	=\mathbb{E}\left[\hat{g}_{i,k} \mid
\mathcal{F}_{k}\right]-g_{i}\left(a_{k}\right)\label{bi}
\end{align}
are noise term  and systematic bias respectively.
Then,  {a lemma of SPSA estimator is introduced as follows.}
\begin{lemma}\label{spsa} \cite[Lemma
	4]{NEURIPS2018_47fd3c87} Let Assumption \ref{ass-lip}  
	holds.	
	Then
	the
	SPSA estimator $(\hat{g}_{i,k})_{i\in\mathcal{N}}$ satisfies
	\begin{align}\label{g_delta}
	\mathbb{E}\left[\hat{g}_{i,k}\mid\mathcal{F}_{k}\right]=g_{i,k}^{\delta_k}
	=\nabla_{a_{i,k}^{\delta_k}}
	u_{i}^{\delta_k}(a_{i,k}^{\delta_k},a_{-i,k}^{\delta_k}),
	\end{align}
	where $u_{i}^{\delta_k}(a_{i,k}^{\delta_k},a_{-i,k}^{\delta_k})$ is a 	
	$\delta$-smooth utility
	function\footnote{The $\delta$-smoothed utility
		function
		$u_{i}^{\delta}(a_{i},a_{-i})=\frac{1}{\operatorname{vol}\left(
			\mathbb{B}_{i}\right) \prod_{j \neq i} \operatorname{vol}\left(
			\mathbb{S}_{j}\right)}\\
		\int_{ \mathbb{B}_{i}} \int_{\prod_{j \neq i}
			\mathbb{S}_{j}} u_{i}\left(a_{i}+\delta_{k}\theta_{i,k} ;
		a_{-i}+\delta_{k}\lambda_{-i,k}\right) d \lambda_{1}
		\cdots d \theta_{i} \cdots d \lambda_{N}.$} and $a_{i,k}^{\delta_k}$ is 
	defined in
	equation (\ref{new-pivot}). In addition, we have
	\begin{align}\label{bik}
	\|\zeta_{i,k}\|\leq   L_{i}\sqrt{N}\delta_{k}
	\end{align}
	and the second moment of the noise term $\rho_{i,k+1}$ is
	$\mathcal{O}\left(1 / \delta_{k}^{2}\right).$
\end{lemma}

\begin{rem}
	When the perturbation radius $\delta_{k} \rightarrow 0,$ the bias will decrease
	to zero, but the noise will increase to
	infinity. Therefore, there is a bias-variance tradeoff between the bias and noise
	variance. Thus, the perturbation radius $\delta_{k}$  should be  selected
	carefully.
	
\end{rem}

In the following, we   present a preliminary  lemma  that
will be used for convergence analysis.

\begin{lemma}\label{lemma-convergence}
	Suppose Assumptions \ref{ass-lip}-\ref{defn-monotone} hold. Define $D_{k}=\sum_{i
		\in  \mathcal{N}}p_{i}^{w-1}D_{i,k}$ {with}
	{$D_{i,k}=\frac{1}{2}\|a_{i,k}- a_i^{\prime}\|^{2}$ for $ a_i^{\prime} \in
		\mathcal{A}_i$}, and
		$\tilde{\gamma}_{i,k}=\gamma_{i,k}-\gamma_{k}p_{i}^{-w}$ with  a
		constant  $w>0$.  Then
	\begin{align}
	\mathbb{E}\left[D_{k+1}\mid \mathcal{F}_{k} \right] 
	&\leq D_{k}
	+ \gamma_{k}{\left\langle g (a^{'}), a_k-a^{'} \right\rangle}
	+\sum_{i \in
		\mathcal{N}}B_{i}L_{i}\sqrt{N}\gamma_{k}\delta_{k}
	\notag\\
	&+\sum_{i \in
		\mathcal{N}}({C_i}+B_{i}L_{i}\sqrt{N}\delta_{k})p_{i}^{w-1}
	\mathbb{E}\left[
	\left|\tilde{\gamma}_{i,k}
	\right|\mid
	\mathcal{F}_{k} \right]
	\notag\\
	&+\sum_{i
		\in
		\mathcal{N}}\frac{G_{i}^{2}p_{i}^{w-1}}{2\delta_{k}^{2}}\mathbb{E}\left[\gamma_{i,k}^{2}\mid\mathcal{F}_{k}
	\right], \quad {\forall a^{'}\in \mathcal{A}}, \notag
	\end{align}
	where ${C_i}$ is some scalar {satisfying  $\left| \left\langle
		g_{i}(a ),a_{i }- a_i^{\prime}\right\rangle \right|
		\leq {C_i}$  for any $a\in \mathcal{A}  $ and $ a_i^{\prime} \in
		\mathcal{A}_i$ ,}
	$B_{i}=\max_{a_{i}^{\prime},a_{i} \in\mathcal{A}_{i}}\|a_{i}^{\prime}-a_{i}\|$
	is the Euclidean diameter of $\mathcal{A}_{i}$,
	$L_{i}$ denotes the Lipschitz constant and 	$G_{i}^{2}=d_{i}^{2}
	\max_{a \in\mathcal{A}}\left|u_{i}(a_{i},a_{-i})\right|^{2}$ is a
	bounded	constant.
\end{lemma}
\begin{proof}[\textbf{Proof}]
	{Note by Algorithm \ref{alg1} and the definition of
		$\sigma-$algebra  that $a_{k}$
		is  adapted to $\mathcal{F}_{k}$. Since
		$\rho_{i,k+1}$ is an $\mathcal{F}_{k}-$adapted martingale difference
		sequence by \eqref{ui},}   we have	 that %for any $ a_i^{\prime} \in
	%\mathcal{A}_i$,
	\begin{align}\label{distance}
	D_{i,k+1}&=D(a_{i,k+1},a_i^{\prime})\notag\\
	&=\frac{1}{2}\|a_{i,k+1}-a_i^{\prime}\|^{2}\mathbf{1}_{\left\{I_{i}^{k}=1\right\}}
	+\frac{1}{2}\|a_{i,k}-a_i^{\prime}\|^{2}\mathbf{1}_{\left\{I_{i}^{k}=0\right\}}\notag\\
	&\overset{(a)}{\leq}\left(
	\frac{1}{2}\|a_{i,k}+\gamma_{i,k}\hat{g}_{i,k}-a_i^{\prime}\|^{2}\right)
	\mathbf{1}_{\left\{I_{i}^{k}=1\right\}}\notag\\
	&\quad\quad+\frac{1}{2}\|a_{i,k}-a_i^{\prime}\|^{2}\mathbf{1}_{\left\{I_{i}^{k}=0\right\}}\notag\\
	&=\left( \frac{1}{2}\|a_{i,k}-a_i^{\prime}\|^{2}+\gamma_{i,k}\left\langle
	\hat{g}_{i,k},a_{i,k}-a_i^{\prime} \right\rangle
	+\frac{1}{2}\gamma_{i,k}^{2}\|\hat{g}_{i,k}\|^{2}\right) \notag\\
	&\quad\quad\mathbf{1}_{\left\{I_{i}^{k}=1\right\}}
	+\frac{1}{2}\|a_{i,k}-a_i^{\prime}\|^{2}\mathbf{1}_{\left\{I_{i}^{k}=0\right\}}
	\notag\\	
	&\overset{{\eqref{es}}}{=}D_{i,k}+\left( \gamma_{i,k}\left\langle
	g_{i}(a_{k})+\rho_{i,k+1}+\zeta_{i,k}, a_{i,k}-a_i^{\prime}\right\rangle
	\right) \mathbf{1}_{\left\{I_{i}^{k}=1\right\}}\notag\\
	&\quad\quad+\frac{1}{2}
	\gamma_{i,k}^{2}\|\hat{g}_{i,k}\|^{2}\mathbf{1}_{\left\{I_{i}^{k}=1\right\}}
	,\quad {\forall a_{i}^{'}\in \mathcal{A}_{i},}
	\end{align}
	where $(a)$ comes from the nonexpansibility of the
	projection.%,  $(b)$ comes from the definition of estimator  $\hat{g}_{i,k}$ in
	%(\ref{es}). 	
	
	Let
	$G_{i}^{2}=d_{i}^{2}\max_{a\in\mathcal{A}}|u_{i}(a_{i},a_{-i})|^{2}$, and we can 
	know that $G_i$ is a
	bounded constant  {because the set $\mathcal{A}$ is  compact  and 
		the 
		function $u_{i}$ is  continuous (Assumption
		\ref{ass-lip}).}
	Thus, with Line 8 of Algorithm \ref{alg1}, we obtain $\|\hat{g}_{i,k}\|^2 \leq
	G_{i}^{2}/\delta_k^2,$ and inequality (\ref{distance}) becomes  	
	\begin{align}\label{31}
	D_{i,k+1}&\leq D_{i,k}
	+(\gamma_{i,k}\left\langle g_{i}(a_{k}),a_{i,k}-a_i^{\prime}
	\right\rangle)\mathbf{1}_{\left\{I_{i}^{k}=1\right\}}\notag\\
	&+(\gamma_{i,k}\left\langle \rho_{i,k+1}, 
	a_{i,k}-a_i^{\prime}\right\rangle
	)\mathbf{1}_{\left\{I_{i}^{k}=1\right\}}\notag\\
	&+(\gamma_{i,k}\left\langle \zeta_{i,k},a_{i,k}-a_i^{\prime} \right\rangle
	)\mathbf{1}_{\left\{I_{i}^{k}=1\right\}}\notag\\
	&+\frac{1}{2}\gamma_{i,k}^{2}\frac{G_{i}^{2}}{\delta_{k}^{2}}
	\mathbf{1}_{\left\{I_{i}^{k}=1\right\}}.
	\end{align}
	Expressing 	
	$\gamma_{i,k}=\gamma_{k}p_{i}^{-w}+\tilde{\gamma}_{i,k}$ with
	$\tilde{\gamma}_{i,k}=\gamma_{i,k}-\gamma_{k}p_{i}^{-w}$,
	we
	have  	
	\begin{align}\label{bd-bik}
	&D_{i,k+1}\leq D_{i,k}
	+\left[ \tilde{\gamma}_{i,k}\left\langle
	g_{i}(a_{k}),a_{i,k}-a_i^{\prime}\right\rangle
	\right] \mathbf{1}_{\left\{I_{i}^{k}=1\right\}}\notag\\
	&+\left[ \tilde{\gamma}_{i,k}\left\langle
	\rho_{i,k+1},a_{i,k}-a_i^{\prime}\right\rangle
	\right] \mathbf{1}_{\left\{I_{i}^{k}=1\right\}}
	+\left[ \tilde{\gamma}_{i,k}\left\langle
	\zeta_{i,k},a_{i,k}-a_i^{\prime}\right\rangle
	\right] \mathbf{1}_{\left\{I_{i}^{k}=1\right\}}\notag\\
	&+\frac{\gamma_{k}}{p_{i}^{w}}\left\langle
	g_{i}(a_{k}),a_{i,k}-a_i^{\prime}\right\rangle
	\mathbf{1}_{\left\{I_{i}^{k}=1\right\}}+\frac{\gamma_{k}}{p_{i}^{w}}\left\langle 
	\rho_{i,k+1},a_{i,k}-a_i^{\prime}
	\right\rangle
	\mathbf{1}_{\left\{I_{i}^{k}=1\right\}}\notag\\
	&+\frac{\gamma_{k}}{p_{i}^{w}}\left\langle \zeta_{i,k},a_{i,k}-a_i^{\prime}
	\right\rangle
	\mathbf{1}_{\left\{I_{i}^{k}=1\right\}}
	+\frac{1}{2}\gamma_{i,k}^{2}\frac{G_{i}^{2}}{\delta_{k}^{2}}
	\mathbf{1}_{\left\{I_{i}^{k}=1\right\}} \notag\\
	& \leq D_{i,k}
	+\left[ \left| \tilde{\gamma}_{i,k}\right| \left| \left\langle
	g_{i}(a_{k}),a_{i,k}-a_i^{\prime}\right\rangle \right| \right]
	\mathbf{1}_{\left\{I_{i}^{k}=1\right\}}\notag\\
	&+\left[ \left( \tilde{\gamma}_{i,k}\right) \left\langle
	\rho_{i,k+1},a_{i,k}-a_i^{\prime}\right\rangle
	\right] \mathbf{1}_{\left\{I_{i}^{k}=1\right\}}\notag\\
	&+\left[ \left| \tilde{\gamma}_{i,k}\right| \left| \left\langle
	\zeta_{i,k},a_{i,k}-a_i^{\prime}\right\rangle
	\right|\right]  \mathbf{1}_{\left\{I_{i}^{k}=1\right\}}\notag\\
	&+\frac{\gamma_{k}}{p_{i}^{w}}\left\langle
	g_{i}(a_{k}),a_{i,k}-a_i^{\prime}\right\rangle
	\mathbf{1}_{\left\{I_{i}^{k}=1\right\}}+\frac{\gamma_{k}}{p_{i}^{w}}\left\langle 
	\rho_{i,k+1},a_{i,k}-a_i^{\prime}
	\right\rangle
	\mathbf{1}_{\left\{I_{i}^{k}=1\right\}}\notag\\
	&+\frac{\gamma_{k}}{p_{i}^{w}}\left| \left\langle
	\zeta_{i,k},a_{i,k}-a_i^{\prime}
	\right\rangle \right|
	\mathbf{1}_{\left\{I_{i}^{k}=1\right\}}
	+\frac{1}{2}\gamma_{i,k}^{2}\frac{G_{i}^{2}}{\delta_{k}^{2}}
	\mathbf{1}_{\left\{I_{i}^{k}=1\right\}}.
	\end{align}

	Note that $B_i
	=\max_{a_{i}^{\prime},a_{i}\in\mathcal{A}_{i}}\|a_{i}^{\prime}-a_{i}\|<\infty$
	by  compactness of $\mathcal{A}_{i}$. Then by  using \eqref{bik}, we obtain that
	\begin{align}
	\left| \left\langle \zeta_{i,k}, a_{i,k}-a_i^{\prime} \right\rangle\right|
	\leq
	\|\zeta_{i,k}\|\|a_{i,k}-a_i^{\prime}\|\leq B_{i}L_{i}\sqrt{N}\delta_{k},
	\end{align}
	where the first inequality comes from the Cauchy--Schwarz inequality.	Since 
	$\mathcal{A}_{i}$ is a compact convex set  and $g_{i}(a )$ is
	$L_i$-Lipschitz continuous (Assumption \ref{ass-lip}) for all $i$, we  have
	$\left|
	\left\langle
	g_{i}(a_{k}),a_{i,k}-a_i^{\prime}\right\rangle \right|
	\leq {C_i}$ for {any $k\geq 1$}.
	By noting that $a_{i,k}$, $g_{i}(a_{k})$ and $\zeta_{i,k}$ are finite-valued
	$\mathcal{F}_{k}-$measurable random variables,  $\rho_{i,k+1}\in \left\lbrace
	\mathcal{F}_{k},
	\lambda_{i,k} \right\rbrace $.
	%	$\gamma_{i,k}$ in scaenario 1 is equal to
	%	$1/(k^{b}p_{i}^{b})$ with $b=1$, in scenario 2 is
	%	$\mathcal{F}_{k+1}-$measurable.
	Taking conditional expectations on
	$\mathcal{F}_{k}$ on both sides of the inequality \eqref{bd-bik}, {and using
		$\mathbb{E}\left[I_{i}^{k}
		\right]=p_{i}>0 $ for all $i$ and $k$,} we have
	\begin{align}\label{30}
	&\mathbb{E}[D_{i,k+1}\mid \mathcal{F}_{k}]\notag\\
	&\leq D_{i,k} + 
	({C_i}+B_{i}L_{i}\sqrt{N}\delta_{k})\mathbb{E}\left[
	\left|\tilde{\gamma}_{i,k} \right|
	\mathbf{1}_{I_{i}^{k}=1}\mid \mathcal{F}_{k}\right]  \notag\\
	&+\left\langle \mathbb{E}\left[\rho_{i,k+1}\mid \mathcal{F}_{k} \right],
	a_{i,k}-a_i^{\prime} \right\rangle
	\mathbb{E}\left[ \tilde{\gamma}_{i,k}
	\mathbf{1}_{I_{i}^{k}=1}\mid \mathcal{F}_{k}  \right] \notag\\
	&+\frac{\gamma_{k}}{p_{i}^{w-1}}\left\langle g_{i}(a_{k}),
	a_{i,k}-a_i^{\prime}
	\right\rangle  +\frac{\gamma_{k}}{p_{i}^{w-1}}\left\langle 
	\mathbb{E}\left[\rho_{i,k+1}\mid
	\mathcal{F}_{k} \right], a_{i,k}-a_i^{\prime}  \right\rangle \notag\\
	&+\frac{\gamma_{k}}{p_{i}^{w-1}} B_{i}L_{i}\sqrt{N}\delta_{k}
	+\frac{1}{2}\frac{G_{i}^{2}}{\delta_{k}^{2}}
	\mathbb{E}\left[\gamma_{i,k}^{2}\mathbf{1}_{I_{i}^{k}=1}\mid
	\mathcal{F}_{k} \right],
	\end{align}
	where the inequality comes from the independence of random variables 
	$\lambda_{i,k}$
	and $I_{i,k}$ (Assumption 2) with respect to $\mathcal{F}_{k}$.
	
	By the definition of 
	$\rho_{i,k+1}=\hat{g}_{i,k}-\mathbb{E}\left[\hat{g}_{i,k} 
	\mid
	\mathcal{F}_{k} \right] $, %and the law of total expectation (cf. Durrett R.
	%2019, pp. 227-228, \cite{durrett2019probability}),
	we have $\left\langle
	\mathbb{E}\left[\rho_{i,k+1} \mid
	\mathcal{F}_{k} \right], a_{i,k}-a_i^{\prime}  \right\rangle=0 $. We further
	amplify (\ref{30})  by
	removing the indicator function to achieve
	%	inequality
	%	(\ref{30}) becomes
	%	\begin{align}\label{d1}
	%	&\mathbb{E}\left[D_{i,k+1}\mid \mathcal{F}_{k} \right]
	%	\notag\\
	%	&\leq D_{i,k} +({C_i}+B_{i}L_{i}\sqrt{N}\delta_{k})\mathbb{E}
	%	\left[\left|\tilde{\gamma}_{i,k}
	%	\right|\mathbf{1}_{I_{i}^{k}=1} \mid \mathcal{F}_{k}
	%	\right] \notag \\
	%	&+\frac{\gamma_{k}}{p_{i}^{w-1}}\left\langle g_{i}(a_{k}),
	%	a_{i,k}-a_i^{\prime}
	%	\right\rangle +\frac{\gamma_{k}}{p_{i}^{w-1}}B_{i}L_{i}\sqrt{N}\delta_{k}	
	%	\notag\\
	%	&+\frac{1}{2}\frac{G_{i}^{2}}{\delta_{k}^{2}}
	%	\mathbb{E}\left[\gamma_{i,k}^{2}\mathbf{1}_{I_{i}^{k}=1}\mid\mathcal{F}_{k}
	%	\right]. 	
	%	\end{align}
	%	
	%Since
	%	Moreover,
	%	$\mathbf{1}_{\{I_{i}^{k}=1\}}$ is the indicator function of the event that
	%node
	%	$i$
	%	updates its action at instance $t$, and $\mathbb{E}\left[I_{i}^{k}
	%	\right]=p_{i}>0 $
	%	for all $i$ and $k$. Therefore, we can further amplify (\ref{d1}) by
	%	removing the indicator function, and get
	\begin{align}\label{d}
	&\mathbb{E}\left[D_{i,k+1}\mid \mathcal{F}_{k} \right] \notag\\
	&\leq
	D_{i,k}+\frac{1}{2}\frac{G_{i}^{2}}{\delta_{k}^{2}}
	\mathbb{E}\left[\gamma_{i,k}^{2}\mid\mathcal{F}_{k} \right] 
	+\frac{\gamma_{k}}{p_{i}^{w-1}}\left\langle g_{i}(a_{k}),
	a_{i,k}-a_i^{\prime}
	\right\rangle
	\notag\\ 
	&
	+\frac{\gamma_{k}}{p_{i}^{w-1}}B_{i}L_{i}\sqrt{N}\delta_{k}
	+(C_i+B_{i}L_{i}\sqrt{N}\delta_{k})\mathbb{E}
	\left[\left|\tilde{\gamma}_{i,k} \right| \mid \mathcal{F}_{k}
	\right]. 	
	\end{align}
	From (\ref{d}) {it follows that}
	\begin{align} \label{49}
	p_{i}^{w-1}&\mathbb{E}\left[D_{i,k+1} \mid \mathcal{F}_{k} \right]
	\notag\leq
	p_{i}^{w-1}D_{i,k}+\frac{p_{i}^{w-1}}{2}\frac{G_{i}^{2}}{\delta_{k}^{2}}
	\mathbb{E}\left[\gamma_{i,k}^{2}\mid\mathcal{F}_{k}\right]
	\\ &+({C_i}+B_{i}L_{i}\sqrt{N}\delta_{k})p_{i}^{w-1}
	\mathbb{E}\left[\left| \tilde{\gamma}_{i,k} \right|\mid\mathcal{F}_{k}
	\right] \notag\\
	&+\gamma_{k}\left\langle g_{i}(a_{k}) ,
	a_{i,k}-a_i^{\prime}\right\rangle
	+B_{i}L_{i}\sqrt{N}\gamma_{k}\delta_{k}.
	\end{align}
	
	{With the definition
		$D_{k}=\sum_{i \in \mathcal{N}}p_{i}^{w-1}D_{i,k}$, by}  summing up
	(\ref{49})
	from
	$i=1$ to $N$, we obtain
	\begin{align}\label{40}
	\mathbb{E}\left[D_{k+1}\mid \mathcal{F}_{k} \right] 
	&\leq
	D_{k}+\sum_{i \in
		\mathcal{N}}({C_i}+B_{i}L_{i}\sqrt{N}\delta_{k})p_{i}^{w-1}
	\mathbb{E}\left[
	\left|\tilde{\gamma}_{i,k}
	\right|\mid
	\mathcal{F}_{k} \right]  \notag\\
	&+\gamma_{k}\sum_{i \in \mathcal{N}}\left\langle
	g_{i}(a_{k}),
	a_{i,k}{-a_{i}^{'}} \right\rangle
	+\sum_{i \in
		\mathcal{N}}B_{i}L_{i}\sqrt{N}\gamma_{k}\delta_{k} \notag\\
	&
	+\sum_{i \in
		\mathcal{N}}\frac{p_{i}^{w-1}}{2}\frac{G_{i}^{2}}{\delta_{k}^{2}}
	\mathbb{E}\left[\gamma_{i,k}^{2}\mid \mathcal{F}_{k} \right].
	\end{align}
	
	{By Assumption \ref{defn-monotone}, we have
		\begin{align*}
		\sum_{i \in \mathcal{N}}\left\langle g_{i}(a_k)-g_{i}(a^{\prime}),
		a_{i,k}-a_{i}^{\prime} \right\rangle <0, \quad {\forall a^{'}\in
			\mathcal{A},}
		\end{align*}
		which implies
		\begin{align*}
		\sum_{i \in \mathcal{N}}&\left\langle g_{i}(a_k),
		a_{i,k}-a_{i}^{\prime}
		\right\rangle  \leq \sum_{i \in \mathcal{N}}\left\langle
		g_{i}(a^{\prime}),
		a_{i,k}-a_{i}^{\prime}  \right\rangle
		.%\leq 0,
		\end{align*}
		This incorporating with (\ref{40}) proves the lemma.}
\end{proof}

\subsection{Proof of Theorem \ref{regret} }\label{the1}

\begin{proof}[\textbf{Proof}]
	By rearranging the terms of (\ref{d}), we have
	\begin{align}
	\left\langle g_{i}(a_{k}), a_i^{\prime}-a_{i,k} \right\rangle
	&\leq \frac{p_{i}^{w-1}}{\gamma_{k}}\left[D_{i,k}-\mathbb{E}\left[ D_{i,k+1}
	\mid
	\mathcal{F}_{k}\right]  \right] \notag\\
	&
	+\frac{p_{i}^{w-1}}{\gamma_{k}}({C_i}+B_{i}L_{i}\sqrt{N}\delta_{k})
	\mathbb{E}\left[\left|\tilde{\gamma}_{i,k} \right|\mid
	\mathcal{F}_{k}  \right] \notag\\
	& +B_{i}L_{i}\sqrt{N}\delta_{k}
	+\frac{p_{i}^{w-1}}{2\gamma_{k}}\frac{G_{i}^{2}}{\delta_{k}^{2}}
	\mathbb{E}\left[\gamma_{i,k}^{2} \mid \mathcal{F}_{k}\right].
	\end{align}
	{Then by taking} unconditional expectations on both sides of the 
	{above}
	inequality,
	we obtain
	\begin{align}\label{bd-bik2}
	\mathbb{E}\left[\left\langle g_{i}(a_{k}), a_i^{\prime}-a_{i,k}
	\right\rangle
	\right]
	&\leq \frac{p_{i}^{w-1}}{\gamma_{k}}
	(\mathbb{E}\left[D_{i,k} \right]-\mathbb{E}\left[D_{i,k+1} \right]  )\notag\\
	&
	+\frac{p_{i}^{w-1}}{\gamma_{k}}({C_i}+B_{i}L_{i}\sqrt{N}\delta_{k})
	\mathbb{E}\left[\left|\tilde{\gamma}_{i,k} \right|
	\right]\notag\\
	&
	+B_{i}L_{i}\sqrt{N}\delta_{k}+\frac{p_{i}^{w-1}}{2\gamma_{k}}
	\frac{G_{i}^{2}}{\delta_{k}^{2}}\mathbb{E}\left[\gamma_{i,k}^{2} \right],
	\end{align}
	where the inequality comes from the law of total expectation.
	{With the $\gamma_{i,k}=\gamma_k p_i^{-w}$ and
		$\tilde{\gamma}_{i,k}=\gamma_{i,k}-\gamma_k p_i^{-w}$, we have that
		$\tilde{\gamma}_{i,k}=0$. Then by recalling the definition
		$D_{i,k}=\frac{1}{2}\|a_{i,k}-a_i^{\prime}\|^{2}$ and summing up
		\eqref{bd-bik2}} from $k=1$ to $K$, we obtain
	\begin{align}\label{regret-all}
	&\sum_{k=1}^{K}\mathbb{E}\left[\left\langle g_{i}(a_{k}),
	a_i^{\prime}-a_{i,k}
	\right\rangle  \right] \notag\\
	&\quad\quad \leq \underbrace{ \sum_{k=1}^{K}\frac{p_{i}^{w-1}}{2\gamma_{k}}
		\mathbb{E}\left[\|a_{i,k}-a_i^{\prime}\|^{2}-\|a_{i,k+1}-a_i^{\prime}\|^{2}\right]
	}_{Term 1}
	\notag\\
	%&\quad
	%	
	%+\sum_{k=1}^{K}\frac{p_{i}^{w-1}}{\gamma_{k}}({C_i}+B_{i}L_{i}\sqrt{N}\delta_{k})
	%	\mathbb{E}\left[\left|\tilde{\gamma}_{i,k} \right|  \right]
	%	\notag\\
	&\quad\quad +\sum_{k=1}^{K}B_{i}L_{i}\sqrt{N}\delta_{k}
	+\sum_{k=1}^{K}\frac{p_{i}^{w-1}}{2\gamma_{k}}\frac{G_{i}^{2}}{\delta_{k}^{2}}
	\mathbb{E}\left[\gamma_{i,k}^{2} \right].
	\end{align}
	Note that
	\begin{align}\label{regret-term1}
	Term 1
	&= \frac{p_{i}^{w-1}}{2\gamma_{1}}\mathbb{E}\left[\|a_{i,1}-a_i^{\prime}
	\|^{2}  \right]  -
	\frac{p_{i}^{w-1}}{2\gamma_{K}}\mathbb{E}
	\left[\|a_{i,K+1}-a_i^{\prime}\| ^{2}
	\right]  \notag\\
	&
	+\frac{p_{i}^{w-1}}{2}\sum_{k=2}^{K}\left(
	\frac{1}{\gamma_{k}}-\frac{1}{\gamma_{k-1}}\right)
	\mathbb{E}\left[\|a_{i,k}-a_i^{\prime} \|^{2}  \right] \notag\\
	&\leq \frac{p_{i}^{w-1}B_{i}^{2}}{2\gamma_{1}}+\frac{p_{i}^{w-1}}{2}
	\sum_{k=2}^{K}\left(\frac{1}{\gamma_{k}}-\frac{1}{\gamma_{k-1}}
	\right)B_{i}^{2}
	\notag\\
	& \leq \frac{p_{i}^{w-1}B_{i}^{2}}{2\gamma_{K}},
	\end{align}
	where the second inequality comes from the fact that $B_i
	=\max_{ a_{i}^{\prime},a_{i}\in\mathcal{A}_{i}}\| a_{i}^{\prime}-a_{i}\|$ and
	{that $\{\gamma_k\}$ is non-increasing.}
	{Then by \eqref{regret-all}, we have
		\begin{align}\label{regret-all2}
		\sum_{k=1}^{K}\mathbb{E}&\left[\left\langle g_{i}(a_{k}),
		a_i^{\prime}-a_{i,k}
		\right\rangle  \right]
		\leq  \frac{p_{i}^{w-1}B_{i}^{2}}{2\gamma_{K}}
		\notag\\
		&\quad
		+\sum_{k=1}^{K}B_{i}L_{i}\sqrt{N}\delta_{k}
		+\sum_{k=1}^{K}\frac{p_{i}^{w-1}}{2\gamma_{k}}\frac{G_{i}^{2}}{\delta_{k}^{2}}
		\mathbb{E}\left[\gamma_{i,k}^{2} \right].
		\end{align} }

	Since  ${u}_{i}(\bullet ,a_{-i,k})$ is concave in
	$a_{i}\in\mathcal{A}_{i}$, {with} the definition of
	$g_{i}(a_{k})= \nabla_{a_{i}}u_{i}\left(a_{i,k} , a_{-i,k}\right)$, we have
	\begin{align}\label{first-order}
	{u}_{i}(a_i^{\prime},a_{-i,k})-{u}_{i}(a_{i,k},a_{-i,k})\leq
	g_{i}(a_{k})^\top(a_i^{\prime}-a_{i,k}).
	\end{align}
	Therefore, we have
	\begin{align}\label{ui-vi}
	\mathbb{E}\left[\mathcal{R}eg^{(i)} (K)\right]
	&=\mathbb{E}\left[
	\max_{a_i^{\prime}\in\mathcal{A}}\sum_{k=1}^{K}\left\lbrace
	u_{i}(a_i^{\prime},a_{-i,k})-u_{i}(a_{i,k},a_{-i,k}) \right\rbrace \right]
	\notag\\
	&\leq \mathbb{E}\left[
	\max_{a_i^{\prime}\in\mathcal{A}}\sum_{k=1}^{K}g_{i}(a_{k})^{T}(a_i^{\prime}-a_{i,k})
	\right].
	\end{align}
	{Therefore, by combining inequality (\ref{regret-all2}) and
		(\ref{ui-vi}), and using the Jensen's inequality to interchange the max
		and $\mathbb{E}$ operations, we  prove  Theorem \ref{regret}.}
\end{proof}

\subsection{Proof of Corollary \ref{no-regret}}\label{cor1}

\begin{proof}[\textbf{Proof}]
		By substituting  $\gamma_{i,k}=\gamma_{k}p_{i}^{-w}$ with
		$\gamma_{k}=k^{-b}$, $b>0 $ and $ w=1$  into  Theorem \ref{regret}, we have
		\begin{align}\label{35}
		\mathbb{E}\left[\mathcal{R}eg^{(i)} (K)\right]
		&  \leq
		\frac{B_{i}^{2}}{2}k^{b} +
		\sum_{k=1}^{K}B_{i}L_{i}\sqrt{N}\delta_{k}\notag\\ 
		&+\sum_{k=1}^{K}\frac{G_{i}^{2}}{2p_{i}^{2}}\delta_{k}^{-2}k^{-b}.
		\end{align}
	
	By noting that  $\delta_{k}={\delta_{1}}k^{-c}$ with a constant
	$c>0$, we have
	\begin{align}\label{k-c}
	\sum_{k=1}^{K}\delta_{k}=\sum_{k=1}^{K}\frac{{\delta_{1}}}{k^{c}}
	&\leq
	{\delta_{1}}\left( 1+\int_{k=1}^{K}\frac{1}{k^{c}}d_{k}\right)
	\notag\\
	&\leq
	{\delta_{1}}+\frac{{\delta_{1}}}{1-c}K^{1-c}.
	\end{align}
	In the same way, we have
	\begin{align}\label{delta-k}
	\sum_{k=1}^{K}\delta_{k}^{-2}k^{-b}
	=\sum_{k=1}^{K}\frac{k^{2c-b}}{{\delta_{1}^{2}}}
	\leq \frac{1}{{\delta_{1}^{2}}}
	+\frac{K^{1-b+2c}}{(1-b+2c){\delta_{1}^{2}}}.
	\end{align}
	Then, substituting (\ref{k-c}) and (\ref{delta-k}) into (\ref{35}), we obtain
	\begin{align}
	\mathbb{E}&\left[\mathcal{R}eg^{(i)} (K)\right]
	\leq \frac{B_{i}^{2}}{2}k^{b}
	+ \frac{B_{i}L_{i}\sqrt{N}{\delta_{1}}}{1-c}K^{1-c}
	\notag\\
	& +\frac{G_{i}^{2}}{2p_{i}^{2}{\delta_{1}^{2}}(1-b+2c)}K^{1-b+2c}
	+B_{i}L_{i}\sqrt{N}{\delta_{1}}
	+\frac{G_{i}^{2}}{2p_{i}^{2}{\delta_{1}^{2}}}. \notag
	\end{align}
	
	Thus, $\limsup_{k\rightarrow \infty}\mathcal{R}eg^{(i)}(K)/K=0$
	follows from $0<b<1$ and $0<c<b/2$.

\end{proof}

\subsection{Proof of Theorem \ref{the-con}}\label{theo2}

\begin{proof}[\textbf{Proof}]  Recall that  $\gamma_{i,k}=\gamma_{k}p_{i}^{-w}$  and
		$\tilde{\gamma}_{i,k}=\gamma_{i,k}-\gamma_{k}p_{i}^{-w}=0$  for $w=1$.
	Then,   from Lemma \ref{lemma-convergence} {and let
		$a_{i}^{'}=a_{i}^{*}$}, we
	obtain 	\begin{align}%\label{bd-dk}
	\mathbb{E}\left[D_{k+1}\mid \mathcal{F}_{k} \right]
	&\leq D_{k} + \gamma_{k}\left\langle
	g (a^{*}), a_{k}-a^{*} \right\rangle \notag\\
	&+\underbrace{\sum_{i \in
			\mathcal{N}}B_{i}L_{i}\sqrt{N}\gamma_{k}\delta_{k}}_{Term 1}
	+\underbrace{\sum_{i
			\in
			\mathcal{N}}\frac{G_{i}^{2}}{2p_{i}^{2}}\frac{\gamma_{k}^{2}}{\delta_{k}^{2}}
	}_{Term 2} \notag.
	\end{align}

	From Assumption \ref{ass-stepsize},  we  obtain
	\begin{align}
	\sum_{k=1}^{\infty}(Term 1 + Term 2 )<\infty.
	\end{align}
	By {recalling \eqref{vi}  and applying  the Robbins's convergence theorem,}
	we conclude  that $D_{k}$ converges  to some finite random variable $D_{\infty}$ 
	almost
	surely and
	{\[\sum_{k=1}^{\infty} \gamma_{k}\left\langle
		g(a^{*}), a^{*}-a_{k}\right\rangle<\infty.\]}
	The requirement $\sum_{k=1}^{\infty}\gamma_{k}=
	\infty$ in Assumption \ref{ass-stepsize}
	implies that {$\lim\inf_{k\rightarrow\infty}\left\langle g(a^{*}), 	 a^{*}
		-a_k  \right\rangle =0$.}
	So,  there exists a subsequence
	$\{k_{r}\}$ such that
	$\lim_{r\rightarrow\infty} \left\langle v (a^{*}), 	a^{*}  -a_{k_r}
	\right\rangle=0$.
	Let $\tilde{a} $ be a limit point of  $\{a_{ k_{r}}\}$.
	Then, $ \left\langle g(a^{*}), a^{*}-  \tilde{a}\right\rangle=0 $. Hence
	$\tilde{a}=a^{*}$ by the strict monotonicity of $g(a)$ (Assumption
	\ref{defn-monotone}).  Then
	$D(a_{k_{r}}, a^{*})$ converges a.s. to zero. By recalling that
	$D(a_{k},a^{*})$   converges a.s.,
	we {reach the conclusion} that  
	$D(a_{k},a^{*})\xrightarrow[k\rightarrow\infty]{a.s.}0$.
	Hence,
	$a_{k}\xrightarrow[k\rightarrow\infty]{a.s.}a^{*}$.
\end{proof}

\subsection{Proof of Theorem \ref{the2}}\label{theo3}
%Here, we establish the convergence rate of  ROGD-ub algorithm for a
%$\beta$-strongly monotone game with specific step-sizes.

In this part,  we give the analysis of the convergence rate of Algorithm 1 for
the strongly monotone game.
To begin with, we introduce a lemma from   \cite[Lemma
3]{chung1954stochastic}.
\begin{lemma}\label{an}
	Let $\{x_k\}$  be a non-negative sequence such that
	\begin{align}\label{anpq}
	x_{k+1}\leq x_{k}\left(1-\frac{P}{k^{p}} \right)+\frac{Q}{k^{p+q}},
	\end{align}
	where $0<p\leq 1$, $q>0$, and $P,Q>0$. Then assuming $P>q$ if $p=1$, we have
	\begin{align}\label{anpq2}
	x_{k}\leq\frac{Q}{R}\frac{1}{k^{q}}+o\left( \frac{1}{k^{q}}\right)
	\end{align}
	with $R=P$ if $p<1$ and $R=P-q$ if $p=1$.
\end{lemma}

\begin{proof}[\textbf{Proof of Theorem \ref{the2}}]
	{In the setting of Theorem \ref{the2}, we have $\gamma_{i,k}=\gamma_kp^{-w}$
		with $\gamma_k=k^{-1}, w=1$ and $\tilde{\gamma}_{i,k}=0.$ Let $D_{k}=\sum_{i
			\in 	\mathcal{N}}\frac{1}{2}\|a_{i,k}-a_{i}^{*}\|^{2}$}, since the
	game
	is $\beta$-strongly monotone (Assumption \ref{strongly}), by  (\ref{vi}) we
	have   	
	{\begin{align}\label{vxn}
		 \left\langle g \left(a_k \right), a_{k}-a^*\right\rangle 
		&=\left\langle g\left(a_k \right)-g(a^*),
		a_{k}-a^*\right\rangle  + \left\langle g\left(a^* \right),
		a_k-a^*\right\rangle  \notag\\
		&\leq  -\beta \|a_k-a^*\|^2 	=-2\beta D_{k}.
		\end{align}
		{Then let $a_{i}^{'}$ in (\ref{40}) be replaced by $a_{i}^{*}$,}
		and  by  substituting (\ref{vxn}) into
		(\ref{40}) and taking unconditional expectations,  we obtain
		\begin{align}\label{43}
		\mathbb{E}\left[D_{k+1} \right]
		&\leq   (1- 2\beta\gamma_{k}  )\mathbb{E}\left[D_{k} \right]
		+\sum_{i \in \mathcal{N}}B_{i}L_{i}\sqrt{N}\gamma_{k}\delta_{k}
		\notag\\
		&
		+\sum_{i \in \mathcal{N}}
		\frac{G_{i}^{2}}{2\delta_{k}^{2}}\mathbb{E}\left[\gamma_{i,k}^{2}
		\right].
		\end{align}}
	
	{Since $\gamma_{i,k}=1/(kp_i)$, $\gamma_{k}=1/k$, and
		$\delta_{k}={\delta_{1}}k^{-1/3}$ {with $\delta_{1}<\min_{i}
			r_{i}$}, we  obtain from \eqref{43} that}
	\begin{align}\label{36}
	\mathbb{E}\left[D_{k+1} \right]\leq \left(
	1-\frac{2\beta}{k}\right)\mathbb{E}\left[D_{k} \right]
	+\frac{H_{1}+H_{2}}{k^{\frac{4}{3}}},
	\end{align}
	where {constants} $H_{1}=\sum_{i \in
		\mathcal{N}}\sqrt{N}B_{i}L_{i}{\delta_{1}}$ and
	$H_{2}=\sum_{i \in
		\mathcal{N}}{G_{i}^{2}}/{(2p_{i}^{2}{\delta_{1}^{2}})}$.
	{Then we discuss the constant $\beta$ in the following two cases.}
	
	{\textbf{Case 1:}}
	{		
		When $\beta\geq 1/6$.   By Lemma \ref{an}, 
		\begin{align}\label{bd-rate1}
		\mathbb{E}\left[D_{k} \right]
		&\leq \frac{\sum_{i \in
				\mathcal{N}}\left( \sqrt{N}B_{i}L_{i}{\delta_{1}}
			+\frac{G_{i}^{2}}{2p_{i}^{2}{\delta_{1}^{2}}}\right)
		}{2\beta-\frac{1}{3}}
		\frac{1}{k^{1/3}}
		+o\left(
		\frac{1}{k^{1/3}}\right).
		\end{align}}
	
	{\textbf{Case 2:}}
	When $0<\beta< 1/6$. We  rewrite (\ref{36}) in the following form: 		
	\begin{align}\label{39}
	\mathbb{E}\left[D_{k+1} \right]
	&\leq \Pi_{t=1}^k \left(
	1-\frac{2\beta}{t}\right)\mathbb{E}\left[D_1 \right]
	\notag\\
	&+(H_{1}+H_{2})\left[  \sum_{s=1}^{{k-1}}  \Pi_{t= s+1}^{k} \left(
	1-\frac{2\beta}{t}\right) s^{-\frac{4}{3}} {+k^{-\frac{4}{3}}}\right]
	\notag\\
	&\leq \Pi_{t=1}^k  \exp\left(
	-\frac{2\beta}{t}\right)\mathbb{E}\left[D_1 \right]
	\notag\\
	&+(H_{1}+H_{2})\left[  \sum_{s=1}^{k-1}   \Pi_{t= s+1}^{k} \exp \left(
	-\frac{2\beta}{t}\right) s^{-\frac{4}{3}}  {+k^{-\frac{4}{3}}}\right]
	\notag\\
	&\leq \exp\left(
	-\sum_{t=1}^k\frac{2\beta}{t}\right)\mathbb{E}\left[D_1 \right]
	\notag\\
	&+(H_{1}+H_{2})\left[  \sum_{s=1}^{k-1}  \exp \left(
	-\sum_{t= s+1}^{k}\frac{2\beta}{t}\right) s^{-\frac{4}{3}}
	{	+k^{-\frac{4}{3}}}	
	\right], 	
	\end{align}
	where the  second inequality comes from the fact that $1-x \leq \exp(-x),
	x>0$.
	
	{Since by the integral test and the divergence rate of the harmonic
		series,
		we know
		\begin{align}\label{400}
		\sum_{t=1}^{k}\frac{1}{t} > \int_{t=1}^{k+1}\frac{1}{t}d_{t}
		=\ln(k+1)>\ln(k),
		\end{align}
		and
		\begin{align}
		\sum_{t=1}^{k}\frac{1}{t} \leq 1+\int_{t=1}^{k}\frac{1}{t}d_{t} =1+
		\ln(k).\notag
		\end{align}
		
		Furthermore, 
		\begin{align}\label{41}
		\sum_{t=s+1}^{k}\frac{1}{t}
		=\sum_{t=1}^{k}\frac{1}{t}-\sum_{t=1}^{s}\frac{1}{t}
		&\geq \ln(k+1)-\ln(s)-1	\notag\\
		&>\ln(\frac{k}{s})-1.		
		\end{align}}
	Then, substituting (\ref{400}) and (\ref{41}) into (\ref{39}), we have
	\begin{align}\label{4333}
	&\mathbb{E}\left[D_{k+1} \right]
	\leq \exp\left( -  2\beta \ln(k)  \right)\mathbb{E}\left[D_1 \right]
	\notag\\
	&\quad +(H_{1}+H_{2})\Big[  \underbrace{\sum_{s=1}^{{k-1}}  \exp \left(
		-2\beta
		(\ln(\frac{k}{s}){-1})\right)
		s^{-\frac{4}{3}}}_{Term 1}+ k^{-\frac{4}{3}} \Big].
	\end{align}
	For $Term 1$, we have
	\begin{align}
	Term 1
	&= {\exp(2\beta)}\sum_{s=1}^{k-1}
	(\frac{k}{s})^{ - 2\beta} s^{-\frac{4}{3}} \notag\\
	&={\exp(2\beta) }  k^{ -  2\beta}
	\sum_{s=1}^{k-1}  s^{2\beta-\frac{4}{3}} \notag\\
	&\leq {\exp(2\beta) }  k^{-2\beta}  {1\over 2\beta-\frac{1}{3}}
	s^{2\beta-\frac{1}{3}}\big|_{s=1}^{k-1} \notag\\
	&= {{\exp(2\beta) }  \over {\frac{1}{3}-2\beta}}k^{ -
		2\beta}
	\left[{1-(k-1)^{ 2\beta - \frac{1}{3}}}  \right] \notag\\
	&{\leq {\exp(2\beta) \over \frac{1}{3}-2\beta} k^{-2\beta}.} \notag
	\end{align}
	This together with (\ref{4333})	implies
	\begin{align}
	\mathbb{E}\left[D_{k+1} \right]
	&\leq k^{ -  2\beta}\mathbb{E}\left[D_1 \right]
	{+ (H_{1}+H_{2})k^{-\frac{4}{3}}}
	\notag\\
	&\quad+{(H_{1}+H_{2}){\exp(2\beta) }\over
		{\frac{1}{3}-	2\beta}}k^{-2\beta}
	\notag \\
	&=\mathcal{O}( k^{ -  2\beta} ).
	\end{align}
	
	By	combining \textbf{Case 1} with \textbf{Case 2}, we prove the theorem.
\end{proof}

\subsection{Proof of Theorem \ref{the-con-pi}}\label{theo4}

Recall that when the loss probability of bandit
feedback can be obtained, {the step-size
	$\gamma_{i,k}=1/(k^{b}p_{i}), a>0$ can be directly  substituted into Lemma  2 to
	yield $\tilde{\gamma}_{i,k}\equiv 0$}.
But when  $p_i$ is unknown,
step-size is a function related to the number of updates up to the current time. So
we provide  results about such a step-size  as  follows.

%where
%Lemma
%$\ref{lemma7}$ will be used in the convergence rate analysis in scenario 2, and
%Corollary
%\ref{coro} will be used in the convergence analysis of scenario 2 in section
%\ref{proof}.

\begin{lemma}\cite[Lemma 5]{5719290}\label{lemma7}
	Let $\tilde{\gamma}_{i,k}=\gamma_{i,k}-{1}/{(k^{q}p_{i}^{q})}$,
	step-size $\gamma_{i,k}={1}/{(\Gamma_{i}^{k})^{q}}$ where
	$\Gamma_{i}^{k}=\sum_{t=1}^{k}I_{i}^{t}$, $q\in (1/2, 1]$, and
	$\mathbb{E}[I_{i}^{k}]=p_{i}>0$ for all $i$ and $k$. Then, for any
	$\sigma\in(0,1/2 )$ , and for every $\omega \in
	\Omega$, there exists  a
	sufficiently small constant
	$\epsilon>0$ and
	a sufficiently large
	$\tilde{k}(\omega)=\tilde{k}(\sigma,\epsilon)$ such that we have for all
	$k\geq\tilde{k}(\omega)$
	and $i\in \mathcal{N}$,
	\begin{align*}
	(a) \gamma_{i,k}\leq \frac{2^{q}}{k^{q}p_{i}^{q}};   \quad\quad
	(b) \left|
	\tilde{\gamma}_{i,k} \right|
	\leq \frac{2q\epsilon}{p_{i}^{2}k^{\frac{1}{2}+q-\sigma}}.
	\end{align*}
\end{lemma}

Note that $\tilde{k}(\omega)$ is contingent on the sample path
corresponding to $\sigma$ and $\epsilon$. More precisely, we claim the following:
\begin{align*}
\mathbb{P}\left[\omega: \gamma_{i,k} \leq \frac{2^{q}}{k^{q} p_{i}^{q}}
\text {
	for } k \geqslant \tilde{k}(\omega)\right]=1.
\end{align*}

\begin{proof}[\textbf{Proof of Theorem \ref{the-con-pi}}]
	In the setting, we have
	$\tilde{\gamma}_{i,k}=\gamma_{i,k}-\gamma_k p_{i}^{-q}$,
	where $\gamma_k=k^{-q}$ and $\gamma_{i,k}={1}/{(\Gamma_{i}^{k})^{q}}$.
	Then based on Lemma \ref{lemma-convergence} { and replace 	
		$a_{i}^{'}$ with $a_{i}^{*}$}, we obtain from Lemma
	\ref{lemma7}
	that for any $\sigma\in (0,1/2)$	 and any sufficiently small
	$\epsilon>0$, there exists a   sufficiently large
	$\tilde{k}(\omega)=\tilde{k}(\sigma,\epsilon)$ such that for all
	$k\geq\tilde{k}(\omega)$,
	\begin{align}\label{lem}
	\mathbb{E}\left[D_{k+1}\mid \mathcal{F}_{k} \right] 
	&\leq D_{k}+
	k^{-q}\left\langle
	g (a^{*}), a_{ k}-a^{*} \right\rangle
	\\
	& +\underbrace{\sum_{i \in
			\mathcal{N}}({C_i}+B_{i}L_{i}\sqrt{N}\delta_k)p_{i}^{q-1}
		\frac{2q\epsilon}{p_{i}^{2}k^{\frac{1}{2}+q-\sigma}}
	}_{Term 1}\notag\\
	&+\underbrace{\sum_{i \in
			\mathcal{N}}B_{i}L_{i}\sqrt{N}k^{-q}\delta_{k}}_{Term 2}
	+\underbrace{\sum_{i
			\in
			\mathcal{N}}\frac{G_i p_{i}^{q-1}}{2\delta_k^2}
		\frac{2^{2q}}{k^{2q}p_{i}^{2q}}
	}_{Term 3}.
	\end{align}
	
	Since $q\in(1/2, 1]$  and $\sigma\in(0,1/2)$, we 	have
	$\sum_{k=\tilde{k}+1}^{\infty}  k^{-(\frac{1}{2}+q-\sigma)} <\infty$.  Then by $
	\lim _{k \rightarrow \infty} \delta_{k}=0$,
	we conclude that  $\sum_{k=1}^{\infty}Term 1<\infty,~a.s.$
	Using \eqref{ass-delta}, we achieve  $\sum_{k=1}^{\infty}Term 2<\infty $ and
	$\sum_{k=1}^{\infty}Term 3<\infty .$
	Then by recalling \eqref{vi}  and applying  the Robbins's convergence
	theorem to \eqref{lem}, we have that $D_{k}$ convergences a.s. to some finite 
	random variable $D_{\infty}$ and
	$\sum_{k=1}^{\infty}  k^{-q}\left\langle g(a^{*}),a^{*}-a_k
	\right\rangle<\infty$.
	Since $\sum_{k=1}^{\infty} k^{-q}= \infty$, we have  that
	$\lim\inf_{k\rightarrow\infty}\left\langle
	g(a^{*}),a^{*}-a_k \right\rangle=0$.
	So,  there exists a subsequence
	$\{k_{r}\}$ such that $\lim_{r\rightarrow\infty}\left\langle
	g(a^{*}),a^{*}-a_{k_r} \right\rangle =0$.
	Let $\tilde{a}$ be a limit point of the bounded  sequence $a_{k_{r}}$.
	Then, $\left\langle g ({x}^{*}, a^{*}-\tilde{a})  \right\rangle=0 $. Hence
	$\tilde{a} =a^{*}$ by the strict monotonicity of $g_{i}(a_{i},a_{-i})$
	(Assumption
	\ref{defn-monotone}).  Then $D(a_{k_{r}}, a^{*})$ converges  to
	zero almost surely. By recalling that
	$D(a_{k},a^{*})$   converges almost surely,
	we {{reach the conclusion}} that  
	$D(a_{k},a^{*})\xrightarrow[k\rightarrow\infty]{a.s.}0$.
	Hence, $a_{k}\xrightarrow[k\rightarrow\infty]{a.s.}a^{*}$.
\end{proof}

\section{The Application to Fog Computing }\label{simu}

\subsection{Problem Setting}
The common goal of cloud computing and fog computing is to share resources and
services. Therefore, how to effectively manage and allocate resources has become one
of  the
most important parts of fog computing. We consider a numerical study of the proposed
algorithm for the
resource management game in fog computing with noncooperative service providers.

\begin{figure}[!htb]
	\centering
	\includegraphics[width=3.5in]{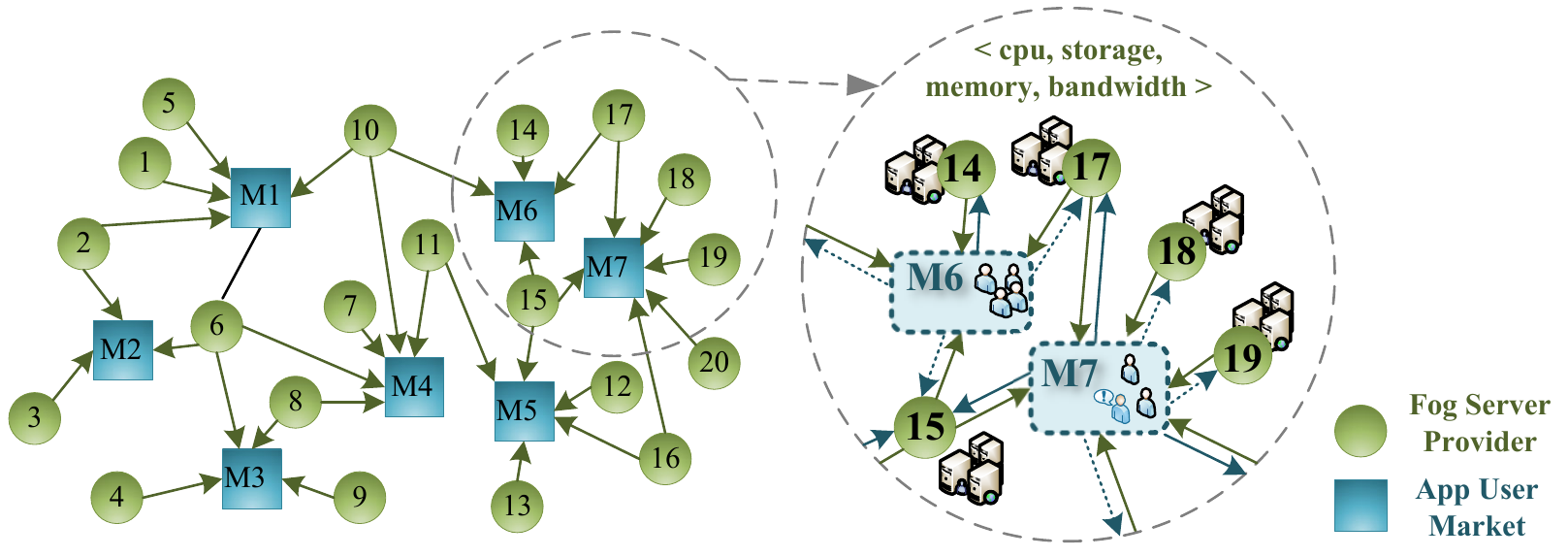}
	\caption{\small Resource
		management  in fog computing.}\label{net}
\end{figure}

Consider 20 fog service  providers (FSPs) and 7 app user markets (AUMs), as shown in 
Figure
\ref{net}. Each FSP can provide
memory, bandwidth, CPU, or storage  to  AUMs.
As a player, each fog server provider needs to determine how many
resources to provide to  app user markets in order to
maximize its own benefits. That is, as a strategy in competition, FSP 
	$i$ 
	provides $a_{i} \in \mathbb{R}^{n_{i}}$ quantity of resources to the $n_{i}$ 
	AUMs it 
	connects.  The connection relationship between the FSP and the AUM is denoted by 
	a matrix $W$, which is a bipartite graph.

The extremely low-information environment is mainly caused by two reasons: cost and
price. The local cost function of FSP $i$ is
$\mathcal{C}_{i}\left(a_{i}\right)$ but the specific form is usually very
complicated. The cost may come from many factors such as hardware, software,
manpower, etc. Operations such as obtaining gradients in such a multi-coupled form
will cause serious computational resource consumption.
Moreover,  what we only know and care about is the value of the cost, so
directly operating on the cost value can reduce the occupation of
computing resources. The price is determined by the relationship between
supply and demand in the market. {\it In real applications, FSP usually cannot know 
the
	specific form of the market pricing function, and in a market-based mechanism,
	resource supply and demand are dynamically changing, {only the value of 
		current price in the market is available to all players.}
	Therefore, after FSP provides some kinds
	of resources to the market, the feedback which can be received from the market is
	their own
	profit value under this strategy.} Overall, FSP compete with each other for 
	market
share to maximize their own profits, that is, in this
networked fog resource
management competition, each FSP
aims to solve
\begin{align*}
\max _{a_{i} \in \mathcal{A}_{i}} (P(W\mathbf{a})-\mathcal{C}_{i}(a_{i}))W_ia_{i}
\end{align*}
given the other providers' profile $\mathbf{a}_{-i}$.

We simulate the two scenarios of known loss probability and unknown 
loss probability, respectively. Throughout this section, the  empirical   
performance  of	OGD-lb  in the expected sense is averaged over 10 paths.
\subsection{Simulations with Known Loss Probability}

We run Algorithm 1 with { $\gamma_{i,k}={1}/{(k^{b}p_{i})}$  and
	$\delta_{k}=k^{-c}$.
	Firstly, we set $b=0.7$, $c=8/25$, $p_{i}=0.6$ and
	display
	the sublinear {expectation-valued} regret  in Figure 
	\ref{fig-no-regret}, which implies that  algorithm OGD-lb meets 
	the 
	no-regret 
	property.  In other words, the online scheme is
	performing at least as well as any static strategy.
	
	\begin{figure}[htbp]
		\centering
		\includegraphics[width=2.3in]{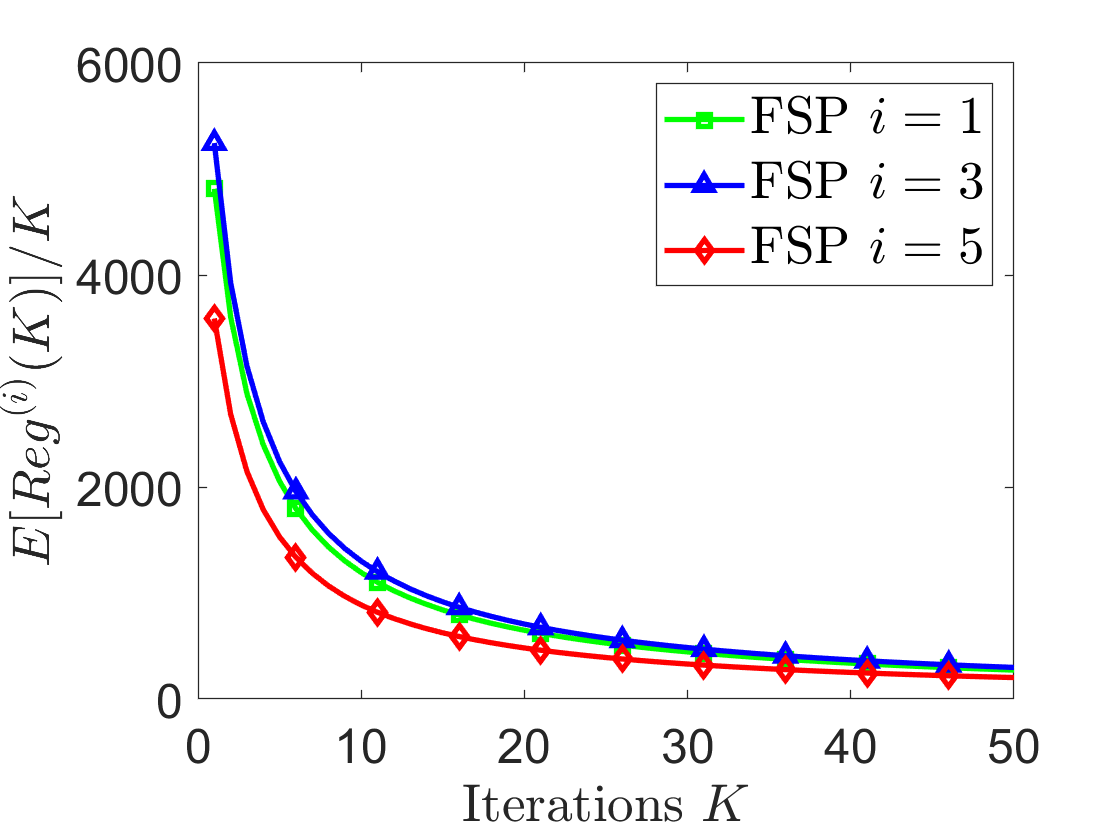}
		\caption{\small {Average regret}
			$\mathbb{E}\left[\mathcal{R}eg^{(i)}(K)\right]/K$ with update 
			probability $p_{i}=0.6, i\in
			\{1,3,5\}$.
		}
		\label{fig-no-regret}
	\end{figure}

	Next, keep $b=0.7$, $c=8/25$ and $p_{i}=0.6$ unchanged. 
	Algorithm 1 is run by a single path and the result is demonstrated 
		in  Figure 
		\ref{x},  which shows that the actions  generated by 
		OGD-lb will converge almost surely to the Nash equilibrium. But 
		due to 
	the lossy
	bandits,
	the	curve will sometimes updated and sometimes unchanged.

	\begin{figure}[htbp]
		\centering
		\includegraphics[width=2.3in]{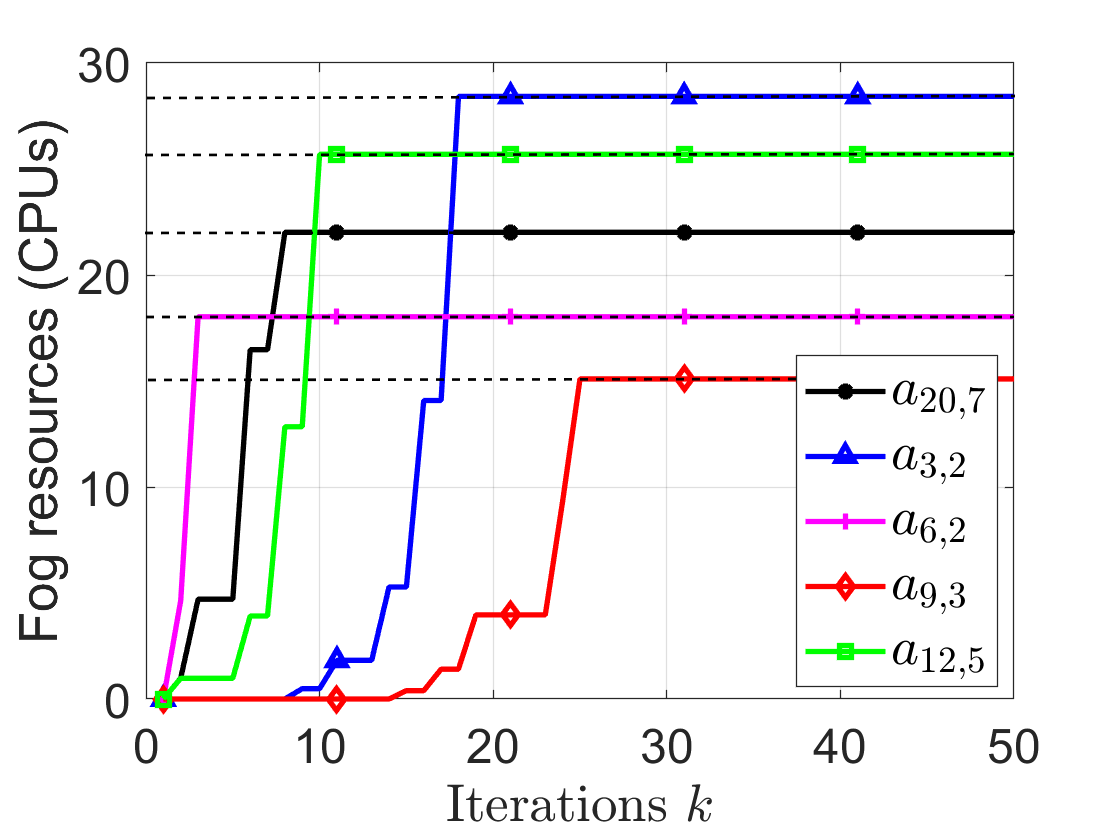}
		\caption{\small The trajectories of the decisions of some fog server
			providers, where
			$a_{i,j}$ represents  the 	amount	of CPUs supplied by FSP $i$
			to AUM $j$, the Nash equilibrium is denoted by the dotted line.}
		\label{x}
	\end{figure}

	{We further keep $c=8/25$   to} explore the influence of
	$p_i$ and   $b$ on the convergence rate of the algorithm.  As shown in Figure
	\ref{rate-p}, the convergence rate increases as $p_{i}$
	decreases. This is because increasing $p_i$ means that the bandit feedback from
	the AUM is more likely to be received by the FSP, that is, the
	algorithm update frequency is increased, and the convergence is accelerated.
	Moreover, we can see from Figure \ref{rate-a} that the convergence
	rate will increase as $b$ decreases. This is because decreasing $b$ will
	increase the update step-size.
	
	\begin{figure}[htbp]
		\centering
		\includegraphics[width=2.3in]{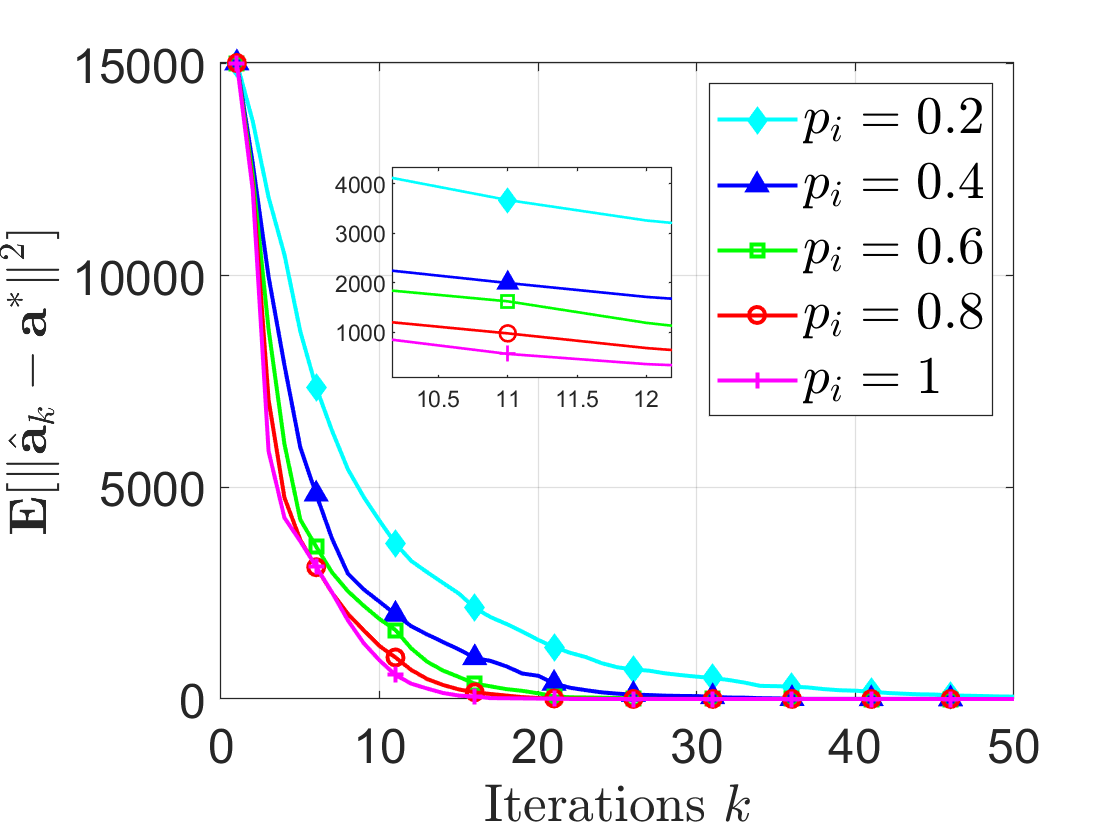}
		\caption{\small The trajectories of $\sum_{i \in \mathcal{N}}
			\mathbb{E}\left[ \|{\hat{a}_{i,k}}-{a^{*}}\|^{2}\right] $ with different
			update probabilities.}
		\label{rate-p}
	\end{figure}
	
	\begin{figure}[htbp]
		\centering
		\includegraphics[width=2.3in]{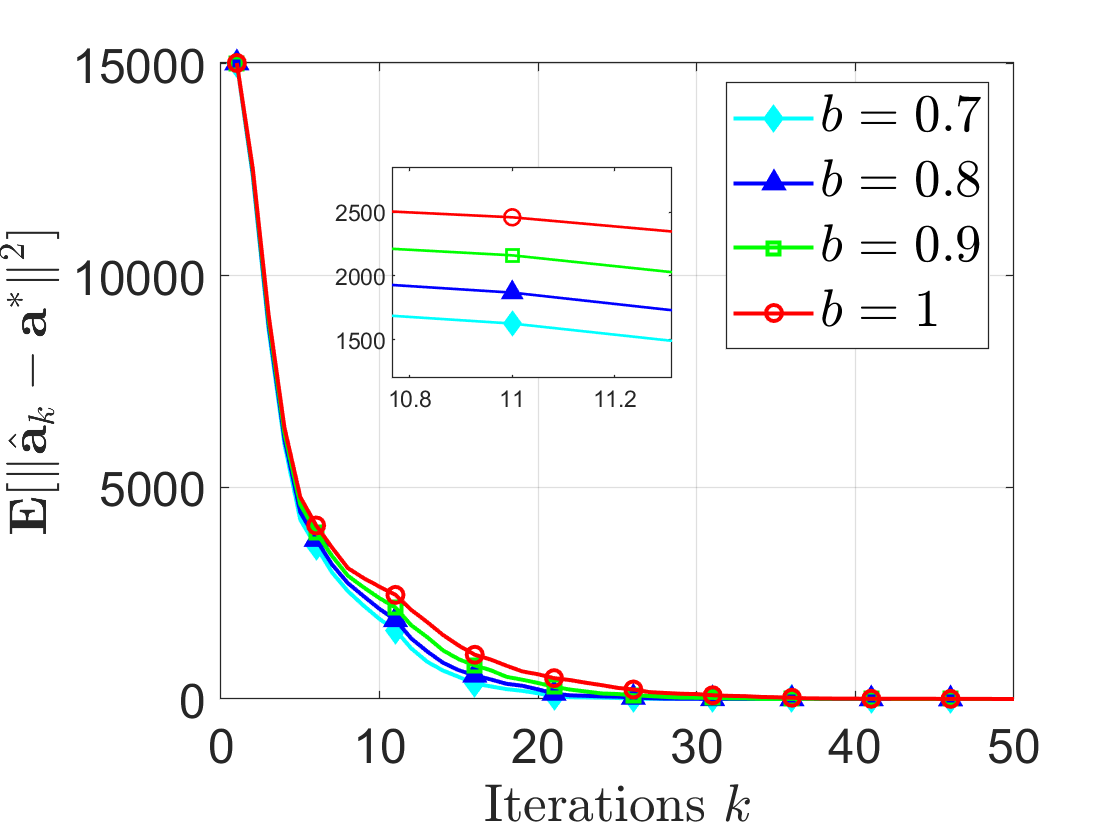}
		\caption{\small The trajectories of $\sum_{i \in \mathcal{N}}
			\mathbb{E}\left[ \|{\hat{a}_{i,k}}-{a^{*}}\|^{2}\right] $ with different
				$b$.}
		\label{rate-a}
	\end{figure}

	Finally, let $\frac{\mathbb{E}\left[\|a_{k}-a^{*}\| \right]}{\|a^{*}\|}\leq
	\varepsilon$ and $p_{i}=P, i\in\mathcal{N}$. We investigate the iterations
	required for the player to reach the specified accuracy $\varepsilon$ under
	different probabilities, and the
	corresponding number of times the feedback information is received.
	It can be seen from Figure \ref{update-bandit} that  when $p_i$ is close to
	$0.8$, the number of iterations required to reach the accuracy $\varepsilon=0.01$
	reaches the
	bottom. Therefore, if human intervention is allowed  in
		applications,  we
		can choose an appropriate update probability (such as $p_{i}=0.8$)
		instead of synchronous updates.
	This will greatly reduce the consumption of
	computing and communication resources.

	\begin{figure}[htbp]
		\centering
		\includegraphics[width=2.3in]{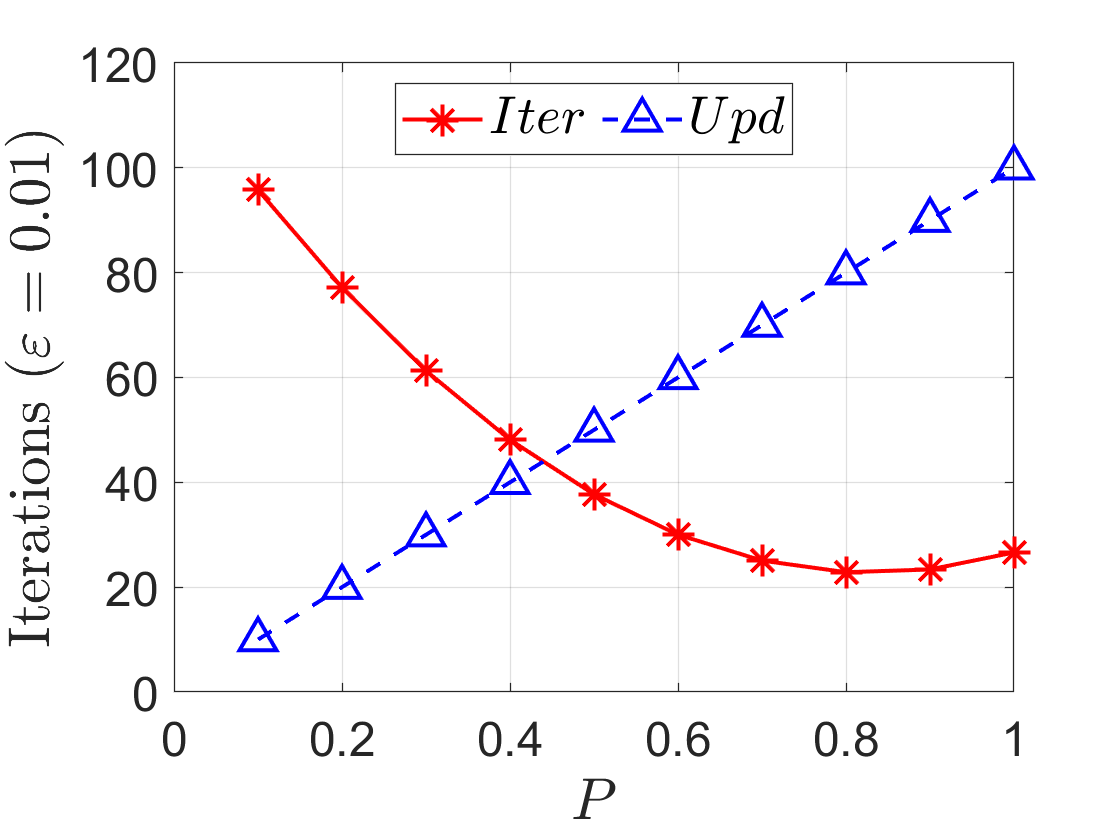}
		\caption{\small Iterations \textit{vs} Update times. (Where
			$p_{i}=P, i\in\mathcal{N}$ and
			$ {\mathbb{E}\left[\|a_{k}-a^{*}\| \right]}/{\|a^{*}\|}\leq
			\varepsilon$.
			\textit{Iter}
			denote the
			iterations required for the players to reach the accuracy of
			$\varepsilon$ under
			different probabilities. \textit{Upd} denote the number of times that
			players update their actions.) }
		\label{update-bandit}
	\end{figure}

	\subsection{Simulations with Unknown Loss Probability}	
	
	Consider Algorithm 1 with  step-size $\gamma_{i,k}={1}/{(\Gamma_{i}^{k})^{q}}$,
	where
	$\Gamma_{i}^{k}=\sum_{t=1}^{k}I_{i}^{t}$, $q\in (1/2, 1]$,
	and perturbation radius $\delta_{k}=k^{-c}$.
	Firstly,  let $q=0.7$ and $c=8/25$. Performing Algorithm 1 with a 
		single path, the result is shown 
		in  Figure 
		\ref{x-pi},  which shows that the actions  generated by 
		OGD-lb converge almost surely to the Nash equilibrium. But 
	due to 
	the lossy
	bandits,
	the	curve will sometimes updated and sometimes unchanged.

	\begin{figure}[htbp]
		\centering
		\includegraphics[width=2.3in]{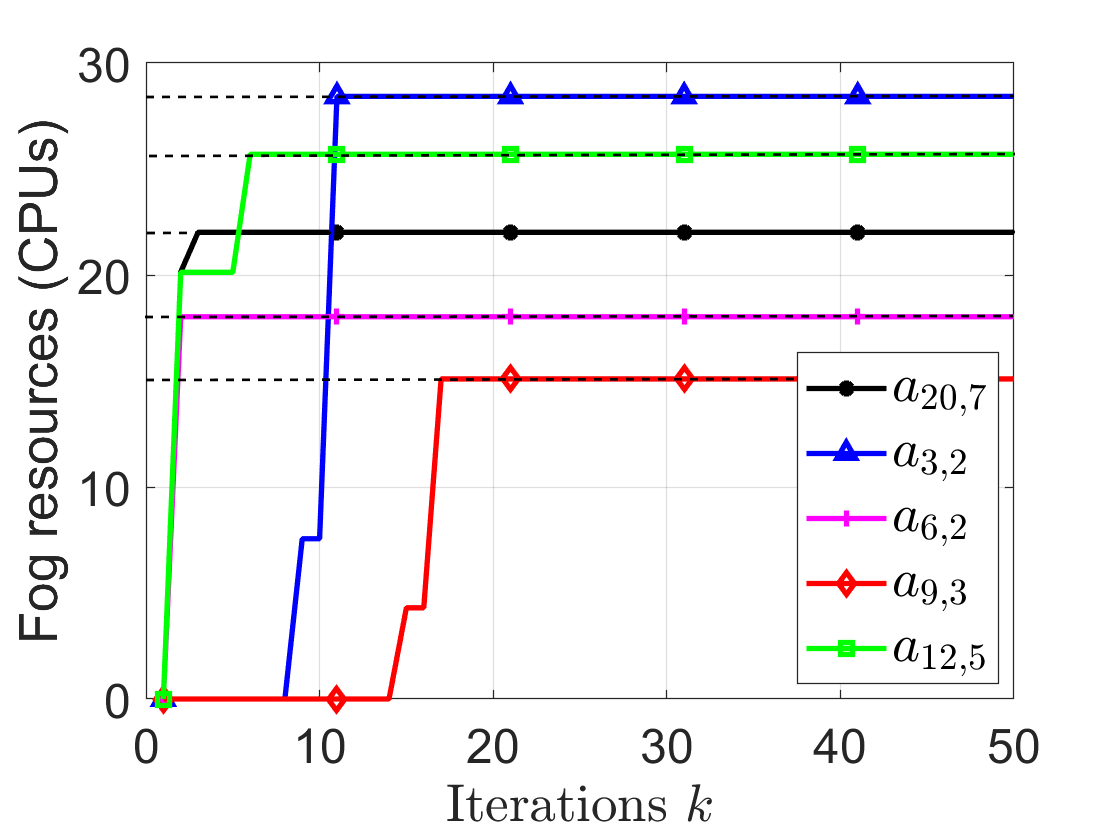}
		\caption{\small The trajectories of the decisions of some fog server
			providers, where
			$a_{i,j}$ represents  the 	amount	of CPUs supplied by FSP $i$
			to AUM $j$, the Nash equilibrium is denoted by the dotted line.}
		\label{x-pi}
	\end{figure}

	Next,  we explore the regret and convergence rate of the algorithm in the
	unknown bandit feedback probability situation through simulations.
	%, which is also a supplement to the theoretical proof. %The  empirical
	%%performance of    OGD-lb  is averaged over 10 paths.
	We set $q=0.7$ and $c=8/25$, and display the {expectation-valued} 
	regret versus the time
	horizon
	$K$ in Figure \ref{fig-no-regret-pi},  which shows that {the 
		average 
		regret}  converges
	sub-linearly, i.e., OGD-lb is a no-regret algorithm.

	\begin{figure}[htbp]
		\centering
		\includegraphics[width=2.3in]{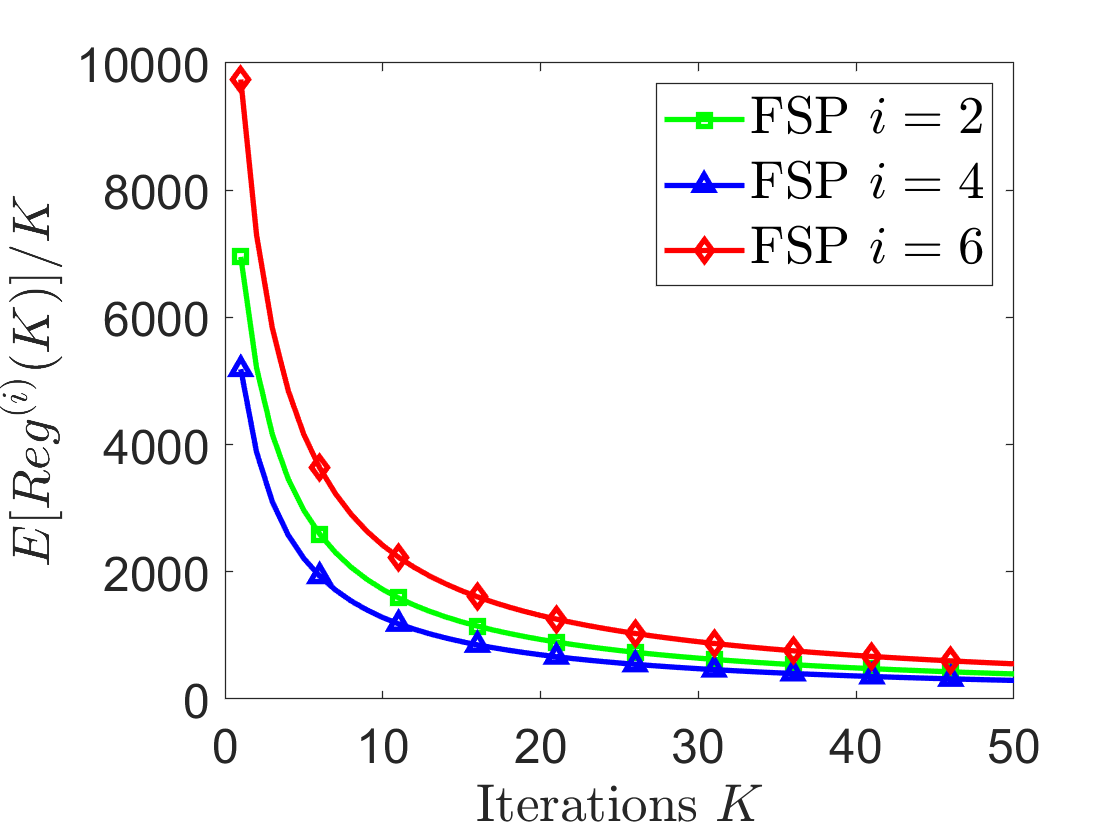}
		\caption{\small {Average regret}
			$\mathbb{E}\left[\mathcal{R}eg^{(i)}(K)\right]/K$ with update 
			probability  $p_{i}=0.8, i\in 
			\{2,4,6\}$.
		}
		\label{fig-no-regret-pi}
	\end{figure}

	Then we set $c=8/25$ and investigate how do  $q$ and $p_i$ influence  the
		algorithm performance.
		It is seen from Figure \ref{rate-a-pi} that the convergence
		rate  increases as $q$ decreases. This is because decreasing $q$ 
	increases the update step-size.
	%Moreover, although we can't get the specific update probability in the
	%	application.  We can do simulation analysis from the perspective of God here
	%		to analyze the convergence rate  under different bandit
	%		feedback loss probabilities.
	%	Just set an update probability artificially, but the probability is
	%unknowable
	%	to the FSP, and the step-size is still the function related to the number
	%	of updates up to the current time.
	We can see from Figure \ref{rate-p-pi} that the
	convergence rate increases as $p_{i}$
	decreases. This is because increasing $p_i$ increases the probability of FSP
	receiving feedback from the AUM, which increases the frequency of
	algorithm updates and accelerates convergence.

	\begin{figure}[htbp]
		\centering
		\includegraphics[width=2.3in]{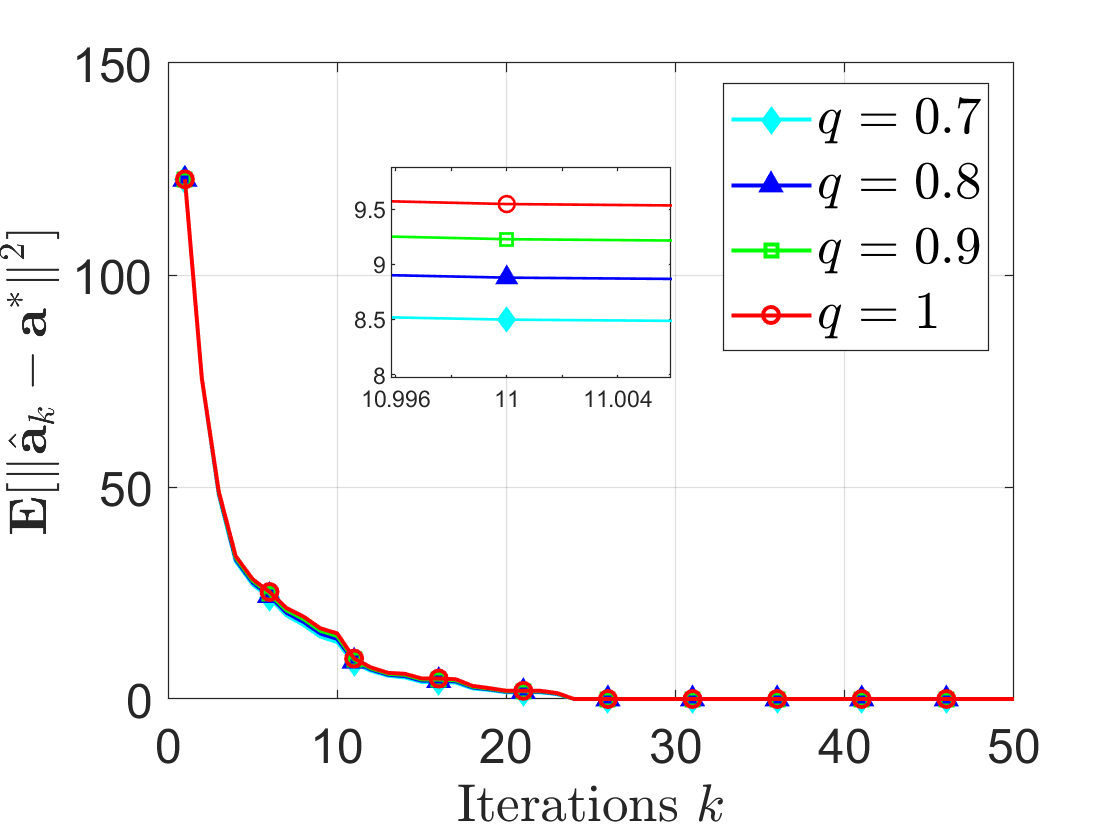}
		\caption{\small The trajectories of $\sum_{i \in \mathcal{N}}
			\mathbb{E}\left[ \|{\hat{a}_{i,k}}-{a^{*}}\|^{2}\right] $ with different
			 $q$.}
		\label{rate-a-pi}
	\end{figure}

	\begin{figure}[htbp]
		\centering
		\includegraphics[width=2.3in]{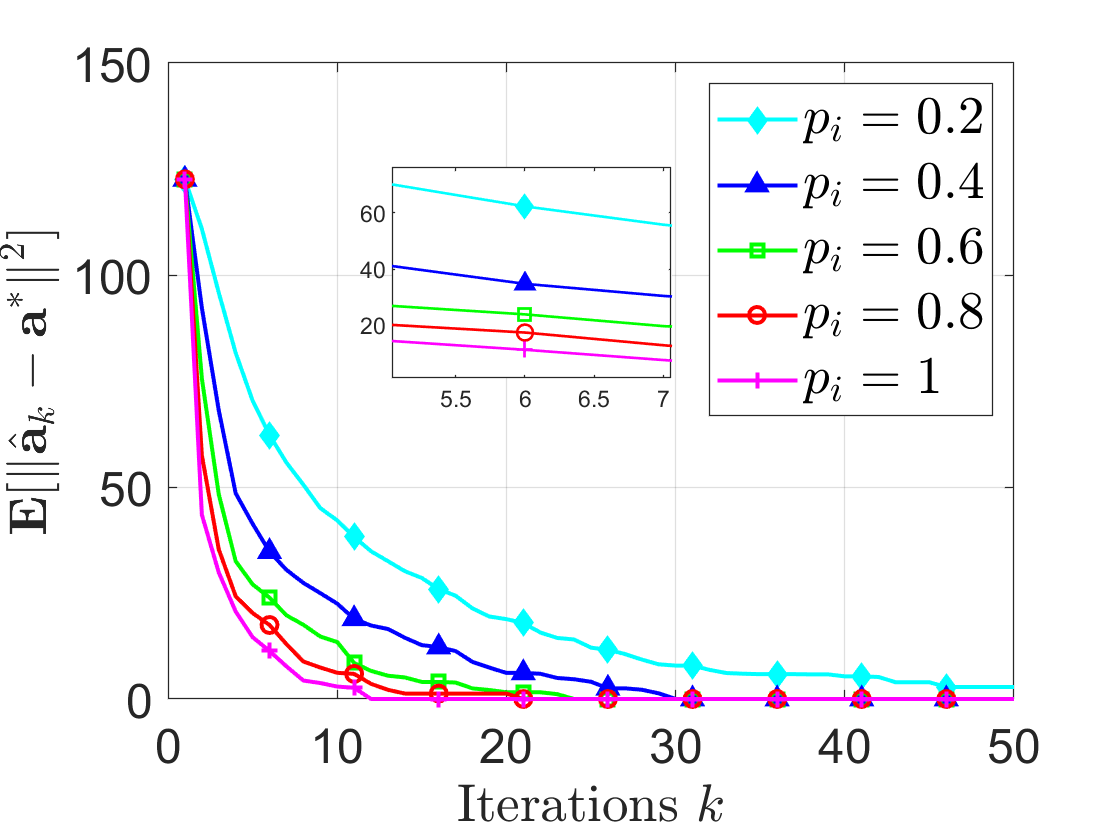}
		\caption{\small The trajectories of $\sum_{i \in \mathcal{N}}
			\mathbb{E}\left[ \|{\hat{a}_{i,k}}-{a^{*}}\|^{2}\right] $ with different
			update probabilities.}
		\label{rate-p-pi}
	\end{figure}

	\section{Conclusion}\label{conclu}
	
	This paper  considered  bandit online learning  for 
		repeated stage games and proposed a novel no-regret algorithm called
	Online Gradient Descent with lossy	bandits (OGD-lb).
	For concave games, 
		we demonstrated that the algorithm meets the no-regret property with a 
		proper 
		selection 
		of step-size. Furthermore, we showed that for strictly
		monotone games,  the actions  generated by OGD-lb can converge to a 
		Nash 
		equilibrium with probability 1
	even when
	the bandit loss probability is unknown.
	Moreover, we  derived an upper
	bound
	of the convergence rate for strongly
	monotone games, which can reach the same order  of the
	algorithm without information loss.
	%	Finally, we show the almost sure convergence of the algorithm   for strictly
	%	monotone games in the case where
	%	the loss probability of bandit feedback is unknown.
	Finally, we applied the
	proposed method to the resource management game in fog computing.



%  \bibliographystyle{elsarticle-num} 
  \bibliographystyle{unsrt}
  \bibliography{ref}

\end{document}